\documentclass{article} 

\usepackage{iclr2019_conference_for_arXiv,times}
\iclrfinalcopy 

\usepackage{hyperref}
\usepackage{url}

\usepackage{subcaption}
\usepackage{amsmath}

\usepackage{algorithm}
\usepackage{algorithmic}
\usepackage{comment}

\usepackage{hyperref}

\usepackage{amsmath,amssymb}
\usepackage{amsbsy}
\usepackage{bm}
\usepackage{amsthm}

\newcommand{\argmin}{\mathop{\mathrm{argmin}}}

\newcommand{\Eqref}[1]{Eq. \eqref{#1}}

\newcommand{\boldone}{{\boldsymbol{1}}}

\newcommand{\calA}{{\mathcal A}}
\newcommand{\calB}{{\mathcal B}}

\newcommand{\calF}{{\mathcal F}}
\newcommand{\calG}{{\mathcal G}}

\newcommand{\calM}{{\mathcal M}}
\newcommand{\calN}{{\mathcal N}}

\newcommand{\calH}{{\mathcal H}}

\newcommand{\fhat}{\widehat{f}}
\newcommand{\fcheck}{\check{f}}

\newcommand{\fbar}{\bar{f}}

\newcommand{\LPi}{L^2} 
\newcommand{\LPiPx}{L^2(P_X)}

\newcommand{\Real}{\mathbb{R}}
\newcommand{\Natural}{\mathbb{N}}

\newcommand{\EE}{\mathrm{E}}
\newcommand{\dd}{\mathrm{d}}
\newcommand{\fstar}{f^{\ast}}

\newcommand{\ftrue}{f^{\mathrm{o}}}

\def\I<#1>{\left\langle #1 \right\rangle}
\def\i<#1>{\left\langle #1 \right\rangle}

\allowdisplaybreaks
\newtheorem{Theorem}{Theorem}
\newtheorem{Lemma}{Lemma}
\newtheorem{Definition}{Definition}

\newtheorem{Proposition}{Proposition}
\newtheorem{Remark}{Remark}

\newcommand{\Well}[1]{W^{(#1)}}
\newcommand{\bell}[1]{b^{(#1)}}

\newcommand{\Holder}{H\"{o}lder }

\newcommand{\Otilde}{\tilde{O}}

\newcommand{\calC}{\mathcal{C}}

\newcommand{\poly}{\mathrm{poly}}

\newcommand{\MB}{M\!B}

\title{Adaptivity of deep ReLU network
for learning in Besov and mixed smooth Besov spaces: \\
optimal rate and curse of dimensionality}

\author{
Taiji Suzuki \\
The University of Tokyo, Tokyo, Japan \\
Center for Advanced Intelligence Project, RIKEN \\
Japan Digital Design \\
\texttt{taiji@mist.i.u-tokyo.ac.jp} \\
}

\begin{document}

\maketitle

\begin{abstract}
Deep learning has shown high performances in various types of tasks from visual recognition to natural language processing,
which indicates superior flexibility and adaptivity of deep learning.
To understand this phenomenon theoretically,
we develop a new approximation and estimation error analysis of 
deep learning with the ReLU activation for functions 
in a Besov space and its variant with mixed smoothness.
The Besov space is a considerably general function space including the \Holder space and Sobolev space,
and especially can capture spatial inhomogeneity of smoothness.
Through the analysis in the Besov space,  
it is shown that deep learning can achieve the minimax optimal rate and 
outperform 
any non-adaptive (linear) estimator such as kernel ridge regression,
which shows that deep learning has higher adaptivity to 
the spatial inhomogeneity of the target function than 
other estimators such as linear ones.
In addition to this, 
it is shown that deep learning can avoid 
the curse of dimensionality
if the target function is in a {\it mixed smooth} Besov space.
We also show that the dependency of the convergence rate on the dimensionality is tight
due to its minimax optimality. 
These results support high adaptivity of deep learning and its superior ability as a feature extractor.
\end{abstract}

\section{Introduction}

Deep learning has shown great success in several applications such as 
computer vision and natural language processing.
As its application range is getting wider, 
theoretical analysis to reveal the reason why deep learning works so well 
is also gathering much attention.
To understand deep learning theoretically, 
several studies have been developed from several aspects such as approximation theory and statistical learning theory.
A remarkable property of neural network is that 
it has universal approximation capability 
even if there is only one hidden layer \citep{cybenko1989approximation,hornik1991approximation,sonoda2015neural}.
Thanks to this property, deep and shallow neural networks can approximate 
any function with any precision (of course, the meaning of the terminology ``any'' must be rigorously defined 
like ``any function in $L^1(\Real)$'').
A natural question coming next to the universal approximation capability is its expressive power.
It is shown that 
the expressive power of deep neural network grows exponentially against the number of layers 
\citep{NIPS2014_5422,bianchini2014complexity,cohen2016expressive,ICML:Cohen+Shashua:2016,NIPS:Poole+etal:2016}
where the ``expressive power'' is defined by several ways.

The expressive power of neural network can be analyzed more precisely 
by specifying the target function's property such as smoothness.
\citet{barron1993universal,barron1994approximation} developed 
an approximation theory for functions having limited ``capacity''
that is measured by integrability of their Fourier transform.
An interesting point of the analysis is that the approximation error is not affected by the dimensionality of the input.
This observation matches the experimental observations that deep learning is quite effective also in high dimensional situations.
Another typical approach is to analyze 
function spaces with smoothness conditions such as the \Holder space. 
In particular, deep neural network with the ReLU activation \citep{nair2010rectified,glorot2011deep} has been extensively studied recently
from the view point of its expressive power and its generalization error.
For example, \cite{DBLP:journals/corr/Yarotsky16} derived
the approximation error of the deep network with the ReLU activation 
for functions in the \Holder space.
\cite{2017arXiv170806633S} evaluated the estimation error of regularized least squared estimator performed by
deep ReLU network based on this approximation error analysis in a nonparametric regression setting.
\cite{petersen2017optimal} generalized the analysis by \cite{DBLP:journals/corr/Yarotsky16} to 
the class of {\it piece-wise} smooth functions.
\cite{arXiv:Imaizumi+Fukumizu:2018} utilized this analysis to derive the estimation error to estimate
the piece-wise smooth function and concluded that deep leaning can outperform linear estimators in that setting;
here, the {\it linear method} indicates an estimator which is linearly dependent on the output observations $(y_1,\dots,y_n)$
(it could be nonlinearly dependent on the input $(x_1, \dots,x_n)$; for example, 
the kernel ridge regression depends on the output observations linearly, but it is nonlinearly dependent on the inputs).
Although these error analyses are standard from a nonparametric statistics view point
and the derived rates are known to be (near) minimax optimal, 
the analysis is rather limited because the analyses are given mainly based on the \Holder space.
However, there are several other function spaces such as the Sobolev space and 
the space of finite total variations.
A comprehensive analysis to deal with such function classes from a unified view point is required.

In this paper, we give generalization error bounds of 
deep ReLU networks for a {\it Besov space} and its variant with {\it mixed smoothness},
which includes the \Holder space, the Sobolev space, and the function class with total variation as special cases. 
By doing so, (i) we show that deep learning achieves the minimax optimal rate on
the Besov space
and notably it outperforms {\it any linear estimator} such as the kernel ridge regression,
and (ii)  we show that deep learning can {\it avoid the curse of dimensionality} 
on the mixed smooth Besov space and achieves the minimax optimal rate.
As related work, \cite{mhaskar1992approximation,mhaskar1993approximation,chui1994neural,mhaskar1996neural,pinkus1999approximation}
also developed an approximation error analysis which essentially leads to analyses for Besov spaces.
However, the ReLU activation is basically excluded and 
comprehensive analyses for the Besov space have not been given. 
Consequently, it has not been clear whether ReLU neural networks can outperform another 
representative methods such as kernel methods.
%
%
As a summary, the contribution of this paper is listed as follows:
\begin{itemize}
\item[(i)] To investigate adaptivity of deep learning, 
we give an explicit form of approximation and estimation error bounds 
for deep learning with the ReLU activation where the target functions are in the Besov spaces ($B_{p,q}^s$) for
$s > 0$ and $0 < p,q \leq \infty$ with $s > d(1/p - 1/r)_+$ where $L^r$-norm is used for error evaluation.
In particular, deep learning outperforms any linear estimator such as kernel ridge regression if 
the target function has highly spatial inhomogeneity of its smoothness. 
See Table \ref{tab:CompDeepLinearApprox} for the overview.

\item[(ii)] To investigate the effect of dimensionality, 
we analyze approximation and estimation problems in so-called the mixed smooth Besov space
by ReLU neural network.
It is shown that deep learning with the ReLU activation can avoid the curse of dimensionality and achieve 
the near minimax optimal rate.
The theory is developed on the basis of the {\it sparse grid} technique \citep{smolyak1963quadrature}.
See Table \ref{tab:Summary} for the overview.
\end{itemize}



\begin{table}[t]
\begin{center}

\caption{
Comparison between the performances achieved by deep learning and linear methods.
Here, $N$ is the number of parameters to approximate a function in a Besov space ($B^s_{p,q}([0,1]^d)$), and $n$ is the sample size.
The approximation error is measured by $L^r$-norm. 
The $\tilde{O}$ symbol hides the poly-log order.
}
\label{tab:CompDeepLinearApprox}
\begin{tabular}{|l|l|l|l|}
\hline
Model   & Deep learning &  Linear method \\   \hline
Approximation error rate & $\tilde{O}(N^{-\frac{s}{d}})$ & $\tilde{O}\left(N^{-\frac{s}{d} + (\frac{1}{p} - \frac{1}{r})_+}\right)$ \\ \hline
Estimation error rate &  $\tilde{O}(n^{-\frac{2s}{2s + d}})$ &  
$\Omega\big(n^{- \frac{2s -  (2/(p \wedge 1) - 1)}{2s + 1 - (2/(p \wedge 1) - 1)}}\big)$ \\ \hline
\end{tabular} \\
\end{center}

\end{table}

\begin{table}[t]
\centering
\caption{Summary of relation between related existing work and our work for 
a mixed smooth Besov space.
$N$ is the number of parameters in the deep neural network,
$n$ is the sample size.
$\beta$ represents the smoothness parameter, and 
$d$ represents the dimensionality of the input.
The approximation accuracy is measured by $L^2$-norm and estimation accuracy is measures by 
the square of $L^2$-norm. 
See Theorem \ref{thm:EstimationErrorNN} for the definition of $u$.
} 
\label{tab:Summary}

\begin{tabular}{|p{2cm}|p{3cm}|p{2.2cm}|p{2.1cm}|p{2.7cm}|}
\hline
Function class   & { \Holder } 
&
Barron class
&
 {
m-Sobolev \phantom{aaa} ($0 < \beta \leq 2$)}
&  {m-Besov \phantom{aaaaaa} ($0 < \beta$)} \\
\hline
\hline
\multicolumn{5}{|l|}{\footnotesize Approximation}  
\\
\hline
Author &  { \small \citet{DBLP:journals/corr/Yarotsky16}, \citet{liang2016deep}}
& \citet{barron1993universal} & 
\small \citet{MontanelliDu2017}
& This work  \\
\hline
Approx. error 
 & $\Otilde(N^{-\frac{\beta}{d}})$ & 
$\Otilde(N^{- 1/2})$ &
$\Otilde(N^{- \beta})$ &
$ \Otilde(N^{-\beta})$\\ 
\hline
\hline
\multicolumn{5}{|l|}{\footnotesize Estimation} 
 \\
\hline
Author  & \small \citet{2017arXiv170806633S} 
& \citet{barron1993universal}
&
\multicolumn{1}{|c|}{----}
 &   This work  \\
\hline
 Estimation error
& $\Otilde(n^{-\frac{2\beta}{2\beta + d}})$ 
& $\Otilde(n^{-\frac{1}{2}})$ 
& \multicolumn{1}{|c|}{----}
&
 $\Otilde(n^{-\frac{2\beta}{2\beta + 1}} \times \log(n)^{\frac{2(d-1)(u + \beta)}{1+2\beta}})$
\\
\hline
\end{tabular}

\end{table}

\section{Set up of function spaces}



\label{sec:FunctionClassDefinitions}

In this section, we define the function classes for which we develop error bounds. 
In particular, we define the Besov space and its variant with mixed smoothness.
The typical settings in statistical learning theory is to estimate a function with 
a {\it smoothness} condition.
There are several ways to characterize ``smoothness.''
Here, we summarize the definitions of representative functional spaces that are appropriate 
to define the smoothness assumption.

Let $\Omega \subset \Real^d$ be a domain of the functions.
Throughout this paper, we employ $\Omega = [0,1]^d$.
For a function $f:\Omega \to \Real$, let $\|f\|_p := \|f\|_{L^p(\Omega)} := (\int_\Omega |f|^p \dd x )^{1/p}$
for $0 < p < \infty$.
For $p=\infty$, we define $\|f\|_\infty :=\|f\|_{L^\infty(\Omega)} := \sup_{x \in \Omega} |f(x)|$.
For $\alpha \in \Real^d$, let $|\alpha| = \sum_{j=1}^d |\alpha_j|$.
Let $\calC^0(\Omega)$ be the set of continuous functions equipped with $L^\infty$-norm:
$\calC^0(\Omega) := \{ f : \Omega \to \Real \mid \text{$f$ is continuous and $\|f\|_\infty < \infty$} \}$
\footnote{Since $\Omega = [0,1]^d$ in our setting, the boundedness automatically follows from the continuity.}.
For $\alpha \in \Natural_+^d$, we denote by 
$
D^\alpha f(x) 
= \frac{\partial^{|\alpha|} f}{\partial^{\alpha_1} x_1 \dots  \partial^{\alpha_d} x_d}(x)
$
\footnote{
We let $\Natural_+ := \{0,1,2,3,\dots\}$, $\Natural_+^d := \{(z_1,\dots,z_d) \mid z_i \in \Natural_+ \}$,
$\Real_+ := \{x \geq 0 \mid x \in \Real\}$, and $\Real_{++} := \{x > 0 \mid x \in \Real\}$.
}.

\begin{Definition}[\Holder space ($\calC^\beta(\Omega)$)]
Let $\beta > 0$ with $\beta \not \in \Natural$ be the smoothness parameter. 
For an  $m$ times differentiable function $f:\Real^d \to \Real$,
let the norm of the \Holder space $\calC^\beta(\Omega)$ be 
$
\|f\|_{\calC^\beta} := \max_{|\alpha| \leq m}
\big\|D^\alpha f\|_{\infty}
+ \max_{|\alpha| = m} \sup_{x,y \in\Omega}
\frac{|\partial^\alpha f(x) - \partial^\alpha f(y)|}{|x-y|^{\beta -m}}, 
$
where $m=\lfloor \beta \rfloor$.
Then,  ($\beta$-)\Holder space $\calC^\beta(\Omega)$ is defined as 
$\calC^\beta(\Omega) = \{f \mid \|f\|_{\calC^\beta} < \infty\}$.
\end{Definition}

The parameter $\beta > 0$ controls the ``smoothness'' of the function.
Along with the \Holder space, the {\it Sobolev space} is also important.
\begin{Definition}[Sobolev space $(W^k_p(\Omega))$] 
Sobolev space $(W^k_p(\Omega))$ with a regularity parameter $k \in \Natural$ and 
a parameter $1 \leq p \leq \infty$ is a set of functions such that 
the Sobolev norm 
$
\textstyle
\|f\|_{W^k_p} := ( \sum_{|\alpha| \leq k}\|D^\alpha f\|^p_{p} )^{\frac{1}{p}}
$ is finite.
\end{Definition}
There are some ways to define a Sobolev space with fractional order, 
one of which will be defined by using the notion of {\it interpolation space} \citep{adams2003sobolev}, 
but we don't pursue this direction here. 
Finally, we introduce {\it Besov space} which further generalizes the definition of the Sobolev space.
To define the Besov space, we introduce the modulus of smoothness.
\begin{Definition}
For a function $f \in L^p(\Omega)$ for some $p \in (0,\infty]$,
the $r$-th modulus of smoothness of $f$ is defined by 
$$
w_{r,p}(f,t) = \sup_{\|h\|_2 \leq t} \|\Delta_h^r(f)\|_{p},
$$
where 
$
\Delta_h^r(f)(x) = 
\begin{cases} \sum_{j=0}^r {r \choose j} (-1)^{r-j} f(x+jh)~& ( x\in \Omega, ~ x+r h\in \Omega), \\
0 & (\text{otherwise}).
\end{cases}
$
\end{Definition}
Based on the modulus of smoothness, the Besov space is defined as in the following definition.
\begin{Definition}[Besov space ($B^\alpha_{p,q}(\Omega)$)]
For $0 < p,q \leq \infty$, $\alpha > 0$, $r:= \lfloor \alpha \rfloor + 1$, 
let the seminorm $|\cdot|_{B^\alpha_{p,q}}$ be 
$$
|f|_{B^\alpha_{p,q}} :=  
\begin{cases}
\left(\int_0^\infty (t^{-\alpha} w_{r,p}(f,t))^q \frac{\dd t}{t} \right)^{\frac{1}{q}} & (q < \infty), \\
\sup_{t > 0} t^{-\alpha} w_{r,p}(f,t)  & (q = \infty).
\end{cases}
$$
The norm of the Besov space $B_{p,q}^\alpha(\Omega)$ can be defined by 
$\|f\|_{B_{p,q}^\alpha} := \|f\|_{p} + |f|_{B^\alpha_{p,q}}$,
and we have $B^\alpha_{p,q}(\Omega) = \{f \in L^p(\Omega) \mid \|f\|_{B_{p,q}^\alpha} < \infty\}$.
\end{Definition}
Note that $p, q < 1$ is also allowed. In that setting, 
the Besov space is no longer a Banach space 
but a quasi-Banach space.
The Besov space plays an important role in several fields such as nonparametric statistical inference \citep{GineNickl2015mathematical} 
and approximation theory \citep{Book:Temlyakov:1993}.
These spaces are closely related to each other as follows \citep{triebel1983theory}: 
\begin{itemize}
\item For $m \in \mathbb{N}$, 
$B^m_{p,1}(\Omega) \hookrightarrow W_p^m(\Omega)   \hookrightarrow B^m_{p,\infty}(\Omega),$ and 
$B^m_{2,2}(\Omega) = W_2^m(\Omega)$. 
\item For $0 < s < \infty$ and $s \not \in \mathbb{N}$, 
$
\calC^s(\Omega) = B^s_{\infty,\infty}(\Omega).
$
\item For $0 < s,p,q,r \leq \infty$ with $s > \delta := d(1/p - 1/r)_+$, it holds that 
$
B_{p,q}^s(\Omega) \hookrightarrow B_{r,q}^{s - \delta}(\Omega).
$
In particular, under the same condition, from the definition of $\|\cdot\|_{B_{p,q}^s}$, it holds that
\begin{align}
B_{p,q}^s(\Omega) \hookrightarrow L^r(\Omega).
\label{eq:BesovLrEmbedding}
\end{align}
\item For $0 < s,p,q \leq \infty$, if $s > d/p$, then 
\begin{align}
B_{p,q}^s(\Omega) \hookrightarrow \calC^0(\Omega).
\label{eq:BesovContEmbedding}
\end{align}
\end{itemize}
Hence, if the smoothness parameter satisfies $s > d/p$, 
then it is continuously embedded in the set of the continuous functions.
However, if $s < d/p$, then the elements in the space are no longer continuous.
Moreover, it is known that $B_{1,1}^1([0,1])$ is included in the space of bounded total variation~\citep{peetre1976new}.
Hence, the Besov space also allows spatially inhomogeneous smoothness with spikes and jumps; which makes difference between linear estimators and deep learning 
(see Sec. \ref{sec:EstErrorAnalysisBesov}).

It is known that 
the minimax rate to estimate $\ftrue$ is lower bounded by  
$
n^{- 2s/(2s + d)},
$
\citep{GineNickl2015mathematical}.
We see that the {\it curse of dimensionality} is unavoidable as long as we consider the Besov space.
This is an undesirable property because we easily encounter high dimensional data in several machine learning problems.
Hence, we need another condition to derive approximation and estimation error bounds that 
are not heavily affected by the dimensionality. 
To do so, we introduce the notion of {\it mixed smoothness}.

%


The Besov space with mixed smoothness is defined as follows \citep{schmeisser1987unconditional,sickel2009tensor}.
To define the space, 
we define the coordinate difference operator as
$$
\Delta_h^{r,i} (f)(x) =\Delta_h^{r}(f(x_1,\dots,x_{i-1},\cdot,x_{i+1},\dots,x_d))(x_i)
$$
for $f:\Real^d \to \Real$, $h \in \Real_+$, $i \in [d]$, and $r \geq 1$. Accordingly, the mixed differential operator for $e \subset \{1,\dots,d\}$ 
and $h \in \Real^d$ is defined as 
$$
\textstyle
\Delta_h^{r,e}(f) = \left(\prod_{i \in e} \Delta_{h_i}^{r,i} \right) (f),~~\Delta_h^{r,\emptyset}(f) = f.
$$
Then, the mixed modulus of smoothness is defined as   
$$
\textstyle
w_{r,p}^e(f,t) := \sup_{|h_i| \leq t_i, i \in e} \|\Delta_h^{r,e}(f)\|_{p}$$ for $t \in \Real_+^d$ and $0 <  p \leq \infty$.
Letting $0 < p,q \leq \infty$, $\alpha \in \Real_{++}^d$ and $r_i := \lfloor \alpha_i \rfloor + 1$,
the semi-norm $|\cdot|_{{\MB_{p,q}^{\alpha,e}}}$ based on the mixed smoothness is defined by
$$
|f|_{\MB_{p,q}^{\alpha,e}} := 
\begin{cases}
\left\{ \int_\Omega [ (\prod_{i \in e} t_i^{-\alpha_i})  w_{r,p}^e(f,t) ]^q \frac{ \dd t}{\prod_{i \in e} t_i} \right\}^{1/q} & (0 < q < \infty), \\
\sup_{t \in \Omega} (\prod_{i \in e} t_i^{-\alpha_i} ) w_{r,p}^e(f,t) 
& (q=\infty).
\end{cases}
$$
By summing up the semi-norm over the choice of $e$, the (quasi-)norm of the mixed smooth Besov space 
(abbreviated to m-Besov space)
is defined by
$$
\|f\|_{\MB^\alpha_{p,q}} := \|f\|_{p} + \sum_{e \subset \{1,\dots,d\}} |f|_{\MB_{p,q}^{\alpha,e}},
$$
and thus $\MB^\alpha_{p,q}(\Omega) := \{f \in L^p(\Omega) \mid \|f\|_{\MB^\alpha_{p,q}} < \infty \}$
where $0 < p,q \leq 1$ and $\alpha \in \Real_{++}^d$.
In this paper, we assume that $\alpha_1 = \dots = \alpha_d$.
With a slight abuse of notation, we also use the notation $MB^\alpha_{p,q}$ for $\alpha > 0$
to indicate $MB^{(\alpha,\dots,\alpha)}_{p,q}$.

For $\alpha \in \Real_+^d$, if $p=q$, the m-Besov space has an equivalent norm with the 
{\it tensor product} of the one-dimensional Besov spaces: 
\begin{align*}
\MB_{p,p}^{\alpha} & =  B_{p,p}^{\alpha_1} \otimes_{\delta_p}  \cdots \otimes_{\delta_p}  B_{p,p}^{\alpha_d}, 
\end{align*} 
where $\otimes_{{\delta_p}}$ is a {\it tensor product with respect to the $p$-nuclear tensor norm}
(see \cite{sickel2009tensor} for its definition and more details).
We can see that the following models are included in the m-Besov space:
\begin{itemize}
\item Additive model \cite{AS:Meier+Geer+Buhlmann:2009}: if $f_j \in B^{\alpha_j}_{p,q}([0,1])$ for $j=1,\dots,d$,
$$f(x) = \sum_{r=1}^d f_d(x_d) \in \MB^\alpha_{p,q}(\Omega),$$
\item Tensor model \cite{TechRepo:Signoretto+etal:2010}: if $f_{r,j} \in B^{\alpha_j}_{p,q}([0,1])$ for $r=1,\dots,R$ and $j=1,\dots,d$,
$$f(x) = \sum_{r=1}^R  \prod_{j=1}^d f_{r,j}(x_j) \in \MB^\alpha_{p,q}(\Omega).$$
(m-Besov space allows $R \to \infty$ if the summation converges with respect to the quasi-norm of 
$\|\cdot\|_{MB_{p,q}^\alpha}$).
\end{itemize}
It is known that an appropriate estimator in these models can avoid curse of dimensionality \citep{AS:Meier+Geer+Buhlmann:2009,raskutti2012minimax,ICML:Kanawaga+etal:2016,suzuki2016minimax}. 
What we will show in this paper supports that this fact is also applied to deep learning from a unifying viewpoint.

The difference between the (normal) Besov space and the m-Besov space can be 
informally explained as follows.
For regularity condition $\alpha_i \leq 2~ (i=1,2)$, the m-Besov space consists of functions for which 
the following derivatives are ``bounded'':
$$
\frac{\partial f}{\partial x_1},\frac{\partial f}{\partial x_2},
\frac{\partial^2 f}{\partial x_1^2},\frac{\partial^2 f}{\partial x_2^2},\frac{\partial^2 f}{\partial x_1 \partial x_2},
\frac{\partial^3 f}{\partial x_1 \partial x_2^2}, \frac{\partial^3 f}{\partial x_1^2 \partial x_2},
\frac{\partial^4 f}{\partial x_1^2\partial x_2^2}.
$$
That is, the ``max'' of the orders of derivatives over coordinates needs to be bounded by 2. 
On the other hand, the Besov space only ensures the boundedness of  the following derivatives:
$$
\frac{\partial f}{\partial x_1},\frac{\partial f}{\partial x_2},
\frac{\partial^2 f}{\partial x_1^2},\frac{\partial^2 f}{\partial x_2^2},\frac{\partial^2 f}{\partial x_1 \partial x_2},
$$
where the ``sum'' of the orders needs to be bounded by 2.
This difference directly affects the rate of convergence of approximation accuracy.
Further details about this space and related topics can be found in a comprehensive survey \citep{dung2016hyperbolic}.

\paragraph{Relation to Barron class.}
\cite{barron99l99lb,barron1993universal,barron1994approximation}
showed that, if the Fourier transform of a function $f$ satisfies some integrability condition, 
then we may avoid curse of dimensionality for estimating neural networks with sigmoidal activation functions.
The integrability condition is given by
$$
\int_{\mathbb{C}^d}  \|\omega \| |\hat{f}(\omega)| \dd \omega < \infty,
$$
where $\hat{f}$ is the Fourier transform of a function $f$.
We call the class of functions satisfying this condition {\it Barron class}.
A similar function class is analyzed by \cite{klusowski2016risk} too.
We cannot compare directly the m-Besov space and Barron class, but they are closely related.
Indeed, if $p=q=2$ and $s = \alpha_1 = \cdots = \alpha_d$, then m-Besov space $\MB^{s}_{2,2}(\Omega)$ is equivalent 
to the tensor product of Sobolev space \cite{sickel2011spline} which 
consists of functions $f: \Omega \to \Real$ satisfying 
$$
\int_{\mathbb{C}^d} \prod_{i=1}^d (1 + |\omega_i|^2)^{s} |\hat{f}(\omega)|^2 \dd \omega < \infty.
$$
Therefore, our analysis gives a (similar but) different characterization of conditions to avoid curse of dimensionality. 

\section{Approximation error analysis} 

In this section, we evaluate how well the functions in the Besov and m-Besov spaces 
can be approximated by neural networks with the ReLU activation.
Let us denote the ReLU activation by $\eta(x) = \max\{x,0\}~(x \in \Real)$, 
and for a vector $x$, $\eta(x)$ is operated in an element-wise manner.
Define the neural network with height $L$, width $W$, sparsity constraint $S$ and norm constraint $B$ as 
\begin{align*}
& \Phi(L,W,S,B) 
:= \{ (\Well{L} \eta( \cdot) + \bell{L}) \circ \dots  
\circ (\Well{1} x + \bell{1})  \\ 
& \mid
\Well{\ell} \in \Real^{W \times W},~\bell{\ell} \in \Real^W,~
\sum_{\ell=1}^L (\|\Well{\ell}\|_0+ \|\bell{\ell}\|_0) \leq S,
\max_{\ell} \|\Well{\ell}\|_\infty \vee \|\bell{\ell}\|_\infty \leq B
\},
\end{align*}
where $\|\cdot\|_0$ is the $\ell_0$-norm of the matrix (the number of non-zero elements of the matrix)
and $\|\cdot\|_\infty$ is the $\ell_\infty$-norm of the matrix (maximum of the absolute values of the elements).
We want to evaluate how large $L,W,S,B$ should be to
approximate $\ftrue \in \MB^\alpha_{p,q}(\Omega)$ by an element $f  \in \Phi(L,W,S,B)$ with precision $\epsilon >0$
measured by $L^r$-norm:
$\min_{f \in \Phi}\|f - \ftrue\|_r \leq \epsilon$.

\subsection{Approximation error analysis for Besov spaces}

Here, we show how the neural network can approximate a function in the Besov space 
which is useful to derive the generalization error of deep learning.
Although its derivation is rather standard as considered in \cite{chui1994neural,bolcskei2017optimal}, 
it should be worth noting that the bound derived here cannot be attained any {\it non-adaptive} method 
and the generalization error based on the analysis is also unattainable by 
any {\it linear} estimators including the kernel ridge regression.
That explains the high adaptivity of deep neural network and how it outperforms usual linear methods such as kernel methods.

To show the approximation accuracy, 
a key step is to show that the ReLU neural network can approximate the {\it cardinal B-spline} with high accuracy.
Let $\calN(x) =  1~  (x \in [0,1]),~0 ~ (\text{otherwise})$, then the 
{\it cardinal B-spline of order $m$} is defined by taking  $m+1$-times convolution of $\calN$:
$$
\calN_m(x) = (\underbrace{\calN * \calN * \dots * \calN}_{\text{$m+1$ times}})(x),
$$
where $f* g(x) := \int f(x -t) g(t) \dd t$. It is known that $\calN_m$ is a piece-wise polynomial of order $m$.
For $k=(k_1,\dots,k_d) \in \Natural^d$ and $j=(j_1,\dots,j_d)\in \Natural^d$, let $M_{k,j}^d(x) = \prod_{i=1}^d \calN_m(2^{k_i} x_i - j_i)$.
Even for $k \in \Natural$, we also use the same notation to express $M_{k,j}^d(x) = \prod_{i=1}^d \calN_m(2^{k} x_i - j_i)$.
Here, $k$ controls the spatial ``resolution'' and $j$ specifies the location on which the basis is put.
Basically, we approximate a function $f$ in a Besov space by a super-position of $M_{k,j}^m(x)$, which is closely related to wavelet analysis \citep{Mallat99a}.

\cite{mhaskar1992approximation,chui1994neural} have shown the approximation ability of 
neural network for a function with bounded modulus of smoothness.
However, the activation function dealt with by the analysis 
does not include ReLU but 
it deals with a class of activation functions satisfying the following conditions,
\begin{align}
& \lim_{x \to \infty} \eta(x)/x^k \to 1,~~\lim_{x \to -\infty} \eta(x)/x^k = 0, ~~ 
\text{$\exists K > 1$ s.t. $|\eta(x)| \leq K (1 + |x|)^k~(x \in \Real)$},
\label{eq:activation_kterm}
\end{align}
for $k=2$ which excludes ReLU. 
\cite{mhaskar1993approximation} analyzed deep neural network under the same setting
but it restricts the smoothness parameter to $s = k+1$.
\cite{mhaskar1996neural} considered the Sobolev space $W^{m}_p$ with an infinitely many differentiable ``bump'' function
which also excludes ReLU.
However,  
approximating the cardinal B-spline by ReLU can be 
attained by appropriately using the technique developed by \cite{DBLP:journals/corr/Yarotsky16} as in the following lemma.

\begin{Lemma}[Approximation of cardinal B-spline basis by the ReLU activation]
\label{lemm:Mnapproximation}
There exists a constant $c_{(d,m)}$ depending only on $d$ and $m$
such that,
for all $\epsilon > 0$,
there exists a neural network 
$
\check{M} \in \Phi(L_0,W_0,S_0,B_0) 
$ with 
$L_0 := 3 +  2 \left\lceil \log_2\left(  \frac{3^{d\vee m} }{\epsilon c_{(d,m)}}  \right)+5 \right\rceil \left \lceil\log_2(d \vee m)\right\rceil$, 
$W_0 := 6 d m(m+2) + 2 d$, $S_0 := L_0 W_0^2$ and  $B_0 := 2 (m+1)^m$
that 
satisfies
$$
\|M_{0,0}^d - \check{M}\|_{L^\infty(\Real^d)} \leq \epsilon,
$$
and $\check{M}(x) = 0$ for all $x \not \in [0,m+1]^d$.

\end{Lemma}

The proof is in Appendix \ref{sec:ProofBsplineApprox}.
Based on this lemma, we can translate several B-spline approximation results 
into those of deep neural network approximation.
In particular, combining this lemma and the B-spline interpolant representations 
of functions in Besov spaces \citep{devore1988interpolation,devore1993wavelet,dung2011optimal},
we obtain the optimal approximation error bound for deep neural networks.
Here, let $U(\calH)$ be the unit ball of a quasi-Banach space $\calH$, and 
for a set $\calF$ of functions, define the worst case approximation error as 
$$
R_r(\calF,\calH) :=\sup_{\ftrue \in U(\calH)}  \inf_{f \in \calF} \|\ftrue - f\|_{L^r([0,1]^d)}.
$$

\begin{Proposition}[Approximation ability for Besov space]
\label{prop:BesovApproxByNN}
Suppose that $0 < p,q,r \leq \infty$ and $0 < s < \infty$ satisfy
the following condition: 
\begin{align}
 s > d(1/p - 1/r)_+.
\label{eq:sprConditions}
\end{align}
Assume that $m \in \Natural$ satisfies $0 < s < \min( m , m -1 + 1/p)$.
Let $\nu = (s - \delta)/(2\delta)$.
For sufficiently large $N \in \Natural$ and $\epsilon = N^{-s/d - (\nu^{-1} + d^{-1})(d/p -s)_+}\log(N)^{-1}$,
let
\begin{flalign*}
\qquad L &= 3 +  2 \lceil \log_2\left(  \frac{3^{d\vee m} }{\epsilon c_{(d,m)}}  \right)+5 \rceil \lceil\log_2(d \vee m)\rceil, 
& W &= N W_0 ,  &\\
\qquad S &= (L-1) W_0^2 N + N, 
&
 B &= O(N^{(\nu^{-1} + d^{-1}) (1\vee (d/p - s)_+) }), &
\end{flalign*}
then it holds that 
$$
R_r(\Phi(L,W,S,B),B^s_{p,q}([0,1]^d))
\lesssim  N^{ - s/d}.
$$

\end{Proposition}

\begin{Remark}
By \Eqref{eq:BesovLrEmbedding}, the condition \eqref{eq:sprConditions} indicates that $\ftrue \in B_{p,q}^s$ satisfies 
$\ftrue \in L^r(\Omega)$.
If we set $p = q = \infty$ and $r=\infty$, then $B_{p,q}^s(\Omega) = C^s(\Omega)$
which yields the result by \cite{DBLP:journals/corr/Yarotsky16} as a special case.
\end{Remark}

The proof is in 
Appendix \ref{sec:ProofPropBesovApp}.
An interesting point is that the statement is valid even for $p \neq r$.
In particular, the theorem also supports non-continuous regime $(s < d/p)$  
in which $L^\infty$-convergence does no longer hold but instead
$L^r$-convergence is guaranteed under the condition $s > d(1/p - 1/r)_+$.
In that sense, the convergence of the approximation error is guaranteed 
in considerably general settings.
\cite{pinkus1999approximation} gave an explicit form of convergence 
when $1 \leq p=r$ for the activation functions satisfying \Eqref{eq:activation_kterm}
which does not cover ReLU and an important setting $p\neq r$.
\cite{petrushev1998approximation} considered $p=r=2$
and activation function with \Eqref{eq:activation_kterm} where $s$ is an integer and $s \leq k+1+(d-1)/2$.
\cite{chui1994neural} and
\cite{bolcskei2017optimal}
dealt with the smooth sigmoidal activation satisfying 
the condition \eqref{eq:activation_kterm} with $k \geq 2$ or a ``smoothed version'' of the ReLU activation which excludes ReLU; 
but \cite{bolcskei2017optimal} presented a general strategy for neural-net approximation by using the notion of best $M$-term approximation.
\cite{mhaskar1992approximation}
gives an approximation bound using the modulus of smoothness,
but the smoothness $s$ and the order of sigmoidal function $k$ in \eqref{eq:activation_kterm} is tightly connected and $\ftrue$ is assumed to be continuous 
which excludes the situation $s < d/p$.
On the other hand, the above proposition does not require such a tight connection
and it explicitly gives the approximation bound for Besov spaces.
\cite{williamson1992splines} derived a spline approximation error bound 
for an element in a Besov space when $d=1$, but the derived bound is 
only $O(N^{-s + (1/p-1/r)_+})$ which 
is the one of non-adaptive methods described below,
and approximation by a ReLU activation network is not discussed.
We may also use the analysis of \cite{cohen2001tree} which is based on compactly supported wavelet bases,
but the cardinal B-spline is easy to handle 
through quasi-interpolant representation as performed in the proof of Proposition \ref{prop:BesovApproxByNN}.

It should be noted that the presented approximation accuracy bound is not trivial 
because it can not be achieved by a {\it non-adaptive method}. 
%
Actually, the {\it linear $N$-width} \citep{tikhomirov1960diameters} of the Besov space is lower bounded as 
\begin{align}
\label{eq:LinearWidthBesov}
\inf_{L_N} \sup_{f \in U(MB_{p,q}^s)} \|f - L_N(f)\|_r 
\gtrsim 
\begin{cases}
N^{- s/d + (1/p - 1/r)_+} & 
\begin{cases}\text{either} & (0 < p \leq r \leq 2), \\ 
\text{or} & (2 \leq p \leq r \leq \infty),  \\
\text{or} & (0 < r \leq p \leq \infty),
\end{cases} \\
N^{- s/d + 1/p - 1/2} & (0 < p <2 <  r  < \infty,~s > d\max(1-1/r,1/p), 
\end{cases}
\end{align}
where the infimum is taken over all linear oprators $L_N$ with rank $N$ from $B_{p,q}^s$ to $L^r$ 
(see \cite{vybiral2008widths} for more details).
Similarly, the {\it best $N$-term approximation error} (Kolmorogov width) of the Besov space is lower bounded as 
\begin{align}
\label{eq:KolmorogororovBesov}
\inf_{S_N \subset B_{p,q}^s} \sup_{f \in U(B_{p,q}^s)} \inf_{\fcheck \in S_N} \|f - \fcheck\|_{L^r(\Omega)} 
\gtrsim 
\begin{cases}
N^{- s/d + (1/p - 1/r)_+} & (1 < p < r \leq 2,~s > d(1/p - 1/r)), \\
N^{- s/d + 1/p - 1/2} & (1 < p <2 <  r \leq \infty,~s > d/p), \\
N^{- s/d } & (2 \leq p < r \leq \infty,~s > d/2),
\end{cases}
\end{align}
if $1 < p < r \leq \infty$, $1 \leq q < \infty$ and $1 < s$, where $S_N$ is any $N$-dimensional subspace of $B_{p,q}^s$
(see \cite{vybiral2008widths}, and see also \cite{Romanyuk2009Kolmogorovwidth,myronyuk2016kolmogorov} for a related space).
That is, any linear/non-linear  approximator with {\it fixed} $N$-bases does not achieve the approximation error 
$N^{- \alpha/d}$ in some parameter settings such as $0 < p < 2 < r $. 
On the other hand, adaptive methods including deep learning can improve the error rate up to $N^{- \alpha/d}$ which is rate optimal
\citep{dung2011optimal}.
The difference is significant when $p < r$.
This implies that deep neural network possesses high adaptivity 
to find which part of the function should be intensively approximated.
In other words, deep neural network can properly extracts the feature of the input (which corresponds to construct an appropriate set of bases)
to approximate the target function in the most efficient way.

%

\subsection{Approximation error analysis for m-Besov space}

Here, we deal with m-Besov spaces instead of the ordinary Besov space.
The next theorem gives the approximation error bound to approximate functions in the m-Besov spaces
by deep neural network models.
Here, define $D_{k,d} := \left(1 + \frac{d-1}{k}\right)^k \left(1 + \frac{k}{d-1}\right)^{d-1}.$
Then, we have the following theorem.

%

\begin{Theorem}[Approximation ability for m-Besov space]
\label{eq:mBesovApproxByNN}
Suppose that $0 < p,q,r \leq \infty$ and $s < \infty$ satisfies
$s > (1/p - 1/r)_+.$
Assume that $m \in \Natural$ satisfies $0 < s < \min( m , m -1 + 1/p)$.
Let $\delta  = (1/p - 1/r)_+$ and $\nu = (s - \delta)/(2\delta)$.
For any $K \geq 1$, let $K^* = \lceil K(1 + \frac{2\delta}{\alpha - \delta})  \rceil $.
Then, for $N =  (2 + (1 - 2^{- \nu})^{-1} ) 2^K D_{K^*,d} $, 
if we set  
\begin{flalign*}
& \textstyle L = 3 +  2 \left\lceil \log_2\left(  \frac{3^{d\vee m} }{ c_{(d,m)}}  \right)+5 
+
(s + (\frac{1}{p} -s)_+ + 1) K^* + \log([e(m+1)]^d (1+K^*))
\right\rceil \lceil\log_2(d \vee m)\rceil,  \\
& W =N W_0, ~~
S = (L-1) N W_0^2 + N,~~
B = O(N^{(\nu^{-1} + 1) (1\vee (1/p - s)_+) }),
\end{flalign*}
then it holds that \\
\begin{subequations}
\label{eq:NeuralBoundmBesov}
\text{(i) For $p \geq r$,}
\begin{flalign}
 &  ~~~~~~~R_r(\Phi(L,W,S,B),MB^s_{p,q}([0,1]^d))
  \lesssim  
2^{-K s } D_{K,d}^{(1/\min(r,1) - 1/q)_+}, 
\end{flalign}
\noindent \text{(ii) For $p < r$,} 
\begin{flalign}
R_r(\Phi(L,W,S,B),MB^s_{p,q}([0,1]^d))
 \lesssim 
\begin{cases}
2^{-K s } D_{K,d}^{(1/r - 1/q)_+} & (r < \infty), \\
2^{-K s } D_{K,d}^{(1 - 1/q)_+}  & (r = \infty).
\end{cases}
\end{flalign}
\end{subequations}

\end{Theorem}

The proof is given in Appendix \ref{sec:ProofmBesovApproxNN}.
Now, the number $S$ of non-zero parameters 
for a given $K$ is evaluated as $S = \Omega(N) \simeq 2^{K} D_{K,d}$ in this theorem. 
It holds that $N \simeq 2^K K^{(d-1)}$, which implies  
$2^{- K} \simeq N^{-1} \log^{d-1}(N)$ if $N\gg d$ (see also the discussion right after Theorem \ref{eq:mBesoApproxSpline}
in Appendix \ref{sec:SparseGridBound} for more details of calculation).
Therefore, when $r \gg q$, the approximation error is given as $O(N^{-s} \log^{s(d-1)}(N))$ in which 
the effect of dimensionality $d$ is much milder than that of Proposition \ref{prop:BesovApproxByNN}.  
This means that the curse of dimensionality is much eased in the mixed smooth space.

The obtained bound is far from obvious. Actually, it is better than any linear approximation methods as follows.
Let the linear $M$-width introduced by \cite{tikhomirov1960diameters} be 
$$
\lambda_N(MB_{p,q}^s,L^r) := \inf_{L_N} \sup_{f \in U(MB_{p,q}^s)} \|f - L_N(f)\|_r,
$$
where the infimum is taken over all linear oprators $L_N$ with rank $N$ from $MB_{p,q}^s$ to $L^r$.
The linear $N$-width of the m-Besov space has been extensively studies as in the following proposition (see Lemma 5.1 of \cite{Complexity:Dung:2011}, and \cite{Romanyuk2001}).

\begin{Proposition}
\label{eq:LinearWidthmBesov}
Let $1 \leq p,r \leq \infty$, $0 < q \leq \infty$ and $s > (1/p - 1/r)_+$.
Then we have the following asymptotic order of the linear width for the asymptotics $N \gg d$: \\
(a) For $p \geq r$,
\begin{align*}
\lambda_N(MB_{p,q}^s,L^r) \simeq 
\begin{cases}
(N^{-1} \log^{d-1}(N))^s & 
{\small
\begin{cases} (q \leq 2 \leq r \leq p < \infty), \\
 (q \leq 1, ~p = r =  \infty), \\
 (1 < p = r \leq 2,~q \leq r), 
\end{cases} 
}
\\
(N^{-1} \log^{d-1}(N))^s (\log^{d-1}(N))^{1/r - 1/q} & (1 < p = r \leq 2,~q > r), \\
(N^{-1} \log^{d-1}(N))^s (\log^{d-1}(N))^{(1/2 - 1/q)_+} & (2 \leq q,~1 < r < 2 \leq p < \infty),
\end{cases}
\end{align*}
(b) For $1 < p < r < \infty$,
\begin{align*}
\lambda_M(MB_{p,q}^s,L^r) \simeq 
\begin{cases}
(N^{-1} \log^{d-1}(N))^{s +  1/r - 1/p} & (2 \leq p,~2 \leq q \leq r),  \\
(N^{-1} \log^{d-1}(N))^{s +  1/r - 1/p}(\log^{d-1}(N))^{(1/r - 1/q)_+} & (r \leq 2).  \\
\end{cases}
\end{align*}

\end{Proposition}

Therefore, the approximation error given in 
Theorem \ref{eq:mBesovApproxByNN} achieves the optimal linear width 
($(N^{-1}\log^{d-1}(N))^s$)
for several parameter settings of $p,q,s$.
In particular, when $p < r$, the bound in Theorem \ref{eq:mBesovApproxByNN} is better than 
that of Proposition \ref{eq:LinearWidthmBesov}.
This is because to prove Theorem \ref{eq:mBesovApproxByNN}, we used an adaptive recovery technique
instead of a linear recovery method.
This implies that, by constructing a deep neural network accurately, we achieve the same approximation accuracy as 
the adaptive one which is better than that of linear approximation.

%
%

\section{Estimation error analysis}
\label{sec:EstErrorAnalysis}
In this section, we connect the approximation theory to generalization error analysis (estimation error analysis).
For the statistical analysis, we assume the following nonparametric regression model:
$$
y_i = \ftrue(x_i) + \xi_i~~~~(i=1,\dots,n),
$$
where $x_i \sim P_X$ with density $0 \leq p(x) < R$ on $[0,1]^d$, and $\xi_i \sim N(0,\sigma^2)$.
The data $D_n = (x_i,y_i)_{i=1}^n$ is independently identically distributed.
We want to estimate $\ftrue$ from the data.
Here, we consider a regularized learning procedure:
$$
\fhat = \argmin_{\bar{f}: f \in \Phi(L,W,S,B)} \sum_{i=1}^n (y_i - \bar{f}(x_i))^2
$$
where $\bar{f}$ is the {\it clipping} of $f$ defined by $\bar{f} = \min\{\max\{f, - F\},F\}$ for $F > 0$
which is realized by ReLU units.
Since the sparsity level is controlled by $S$ and the parameter is bounded by $B$,
this estimator can be regarded as a regularized estimator.
In practice, it is hard to exactly compute $\fhat$.
Thus, we approximately solve the problem by applying sparse regularization such as $L_1$-regularization 
and optimal parameter search through Bayesian optimization.
The generalization error that we present here is an ``ideal'' bound which is valid 
if the optimal solution $\fhat$ is computable. 

\subsection{Estimation error in Besov spaces}
\label{sec:EstErrorAnalysisBesov}

In this subsection, we provide the estimation error rate of deep learning to estimate functions in Besov spaces
by using the approximation error bound given in the previous sections.

\begin{Theorem}
\label{thm:EstimationErrorNNBesov}
Suppose that $0 < p,q \leq \infty$ and $s > d(1/p - 1/2)_+$.
If $\ftrue \in B^{s}_{p,q}(\Omega) \cap L^\infty(\Omega)$ and.  $\|\ftrue\|_{B^{s}_{p,q}} \leq 1$
and $\|\ftrue\|_\infty \leq F$ for $F \geq 1$,  
then 
letting $(W,L,S,B)$ be as in Proposition \ref{prop:BesovApproxByNN} with 
$N \asymp n^{\frac{d}{2s+d}}$, 
we obtain 
\begin{align*}
\EE_{D_n}[\|\ftrue -  \fhat\|_{\LPi(P_X)}^2] \lesssim 
n^{- \frac{2s}{2s + d}} \log(n)^{2},
\end{align*}
where $\EE_{D_n}[\cdot]$ indicates the expectation w.r.t. the training data $D_n$.
\end{Theorem}
The proof is given in Appendix \ref{sec:ProofsOfEstimationErrorBounds}.
The condition $\|\ftrue\|_\infty \leq F$ is required to connect the empirical $L^2$-norm $\frac{1}{n} \sum_{i=1}^n (\fhat(x_i) - \ftrue(x_i))^2$ 
to the population $L^2$-norm $\|\fhat - \ftrue\|_{\LPi(P_X)}^2$.
It is known that the convergence rate $n^{-\frac{2s}{2s + d}}$ is mini-max optimal \citep{donoho1998minimax,GineNickl2015mathematical}. Thus, it cannot be improved by any estimator. Therefore, deep learning can achieve the minimax optimal rate up to $\log(n)^2$-order.
The term $\log(n)^2$ could be improved to $\log(n)$ by using the construction of \cite{petersen2017optimal}.
However, we don't pursue this direction for simplicity.

Here an important remark is that this minimax optimal rate cannot be achieved by any {\it linear estimator}.
We call an estimator {\it linear} when the estimator depends on $(y_i)_{i=1}^n$ linearly
(it can be non-linearly dependent on $(x_i)_{i=1}^n$).
Several classical methods such as the kernel ridge regression, the Nadaraya-Watson estimator and the sieve estimator are included in the class of linear estimators (e.g., kernel ridge regression is given as  $\fhat(x) = k_{x,X}(k_{XX} + \lambda \mathrm{I})^{-1}Y$).
The following proposition given by \cite{donoho1998minimax,zhang2002wavelet} states that the minimax rate of linear estimators is 
lower bounded by $n^{- \frac{2s - 2(1/p - 1/2)_+}{2s + 1 - 2(1/p - 1/2)_+}}$ which is larger than the minimax rate 
$n^{- \frac{2s}{2s + 2}}$ if $p < 2$.
\begin{Proposition}[\cite{donoho1998minimax,zhang2002wavelet}]
\label{Prop:LinearMinimax}
Suppose that $d=1$ and the input distribution $P_X$ is the uniform distribution on $[0,1]$. 
Assume that $s > 1/p$, $1 \leq p, q \leq \infty$ or $s=p=q=1$. 
Then, 
$$
\inf_{\text{$\fhat$: linear}} \sup_{\ftrue \in U(B_{p,q}^s)}\EE_{D_n}[\|\ftrue -  \fhat\|_{\LPi(P_X)}^2]
\gtrsim n^{- \frac{2s - v}{2s + 1 - v}}
$$
where $v = 2/(p \wedge 2) - 1$ and $\fhat$ runs over all linear estimators, that is, 
$\fhat$ depends on $(y_i)_{i=1}^n$ linearly.
\end{Proposition}

When $p < 2$, the smoothness of the Besov space is somewhat inhomogeneous,
that is, a function in the Besov space contains spiky/jump parts and smooth parts
(remember that when $s = p = q = 1$ for $d=1$, the Besov space is included in the set of functions with bounded total variation).
Here, the setting $p < 2$ is the regime where 
there appears difference between non-adaptive methods and deep learning in terms of approximation accuracy
(see 
\Eqref{eq:KolmorogororovBesov}).
On the other hand, the linear estimator captures only global properties of the function 
and cannot capture variability of local shapes of the function.
Hence, the linear estimator cannot achieve the minimax optimal rate if the function has spatially inhomogeneous smoothness.
However, deep learning possesses adaptivity to the spatial inhomogeneity. 

We would like to remark that 
The shrinkage estimator proposed in \cite{donoho1998minimax,zhang2002wavelet} achieves the minimax optimal rate for $s > 1/p$ with $d = 1$
and $1 \leq p,q \leq \infty$ which excludes an interesting setting such as $s = p = q = 1$.
However, the result of Theorem \ref{thm:EstimationErrorNNBesov} also covers more general settings where 
$d \geq 1$ and $s > d(1/p - 1/2)_+$ with $0 < p,q \leq \infty$.

\cite{arXiv:Imaizumi+Fukumizu:2018} has pointed out that such a discrepancy between deep learning and linear estimator
appears when the target function is {\it non-smooth}. 
Interestingly, the parameter setting $s > 1/p$ assumed in Proposition \ref{Prop:LinearMinimax} 
ensures smoothness (see \Eqref{eq:BesovContEmbedding}). 
This means that non-smoothness is not necessarily required to characterize the superiority of deep learning,
but {\it non-convexity} of the set of target functions is essentially important. 
In fact, the gap is coming from the property that the {\it quadratic hull} of the model $U(B_{p,q}^s)$ is strictly larger 
than the original set \citep{donoho1998minimax}.

\subsection{Estimation error in mixed smooth Besov spaces}

Here, we provide the estimation error rate of deep learning to estimate functions in mixed smooth Besov spaces.

\begin{Theorem}
\label{thm:EstimationErrorNN}
Suppose that $0 < p,q \leq \infty$ and $s > (1/p - 1/2)_+$.
Let $u = (1 - 1/q)_+$ for $p \geq 2$ and $u = (1/2 - 1/q)_+$ for $p < 2$.
If $\ftrue \in \MB^{s}_{p,q}(\Omega) \cap L^\infty(\Omega)$ and $\|\ftrue\|_{\MB^{s}_{p,q}} \leq 1$ and $\|\ftrue\|_\infty \leq F$
for $F \geq 1$,  
then 
letting $(W,L,S,B)$ be as in Theorem \ref{eq:mBesovApproxByNN}, 
we obtain 
\begin{align*}
\EE_{D_n}[\|\ftrue -  \fhat\|_{\LPi(P_X)}^2] \lesssim
n^{- \frac{2s}{2s + 1}}\log(n)^{\frac{2(d-1)(u + s)}{1+2s}} \log(n)^2.
\end{align*}
Under the same assumption, if $s > u\log_2(e)$ is additionally satisfied, we also have 
\begin{align*}
\EE_{D_n}[\|\ftrue -  \fhat\|_{\LPi(P_X)}^2] \lesssim 
n^{- \frac{2s - 2u\log_2(e)}{2s + 1 + (1-2u)\log_2(e)}} 
\log(n)^2.
\end{align*}
\end{Theorem}

The proof is given in Appendix \ref{sec:ProofsOfEstimationErrorBounds}.
The risk bound (Theorem \ref{thm:EstimationErrorNN}) indicates 
that the curse of dimensionality can be eased by assuming the mixed smoothness
compared with the ordinary Besov space ($n^{-\frac{2s}{2s + d}}$).
We show that this is almost minimax optimal in Theorem \ref{eq:MinimaxOptimalboundOfmBesov} below.
In the first bound, the dimensionality $d$ comes in the exponent of $\poly \log(n)$ term.
If $u=0$, then the effect of $d$ can be further eased. 
Actually, in this situation ($u=0$), the second bound can be rewritten as 
$$
n^{- \frac{2s}{2s + 1 + \log_2(e)}} \log(n)^2, 
$$
where the effect of the dimensionality $d$ completely disappears from the exponent.
This explains partially why deep learning performs well for high dimensional data.
\cite{MontanelliDu2017} has analyzed the mixed smooth \Holder space  with $s < 2$.
However, 
our analysis is applicable 
to the m-Besov space which is more general than the mixed smooth \Holder space
and the covered range of $s,p,q$ is much larger.

Here, we again remark the adaptivity of deep learning.
Remind that this rate cannot be achieved by the linear estimator for $p < 2$ when $d=1$
by Proposition \ref{Prop:LinearMinimax}. 
This explains the adaptivity ability of deep learning to the spatial inhomogeneity of the smoothness.

\paragraph{Minimax optimal rate for estimating a function in the m-Besov space}

Here, we show the minimax optimality of the obtained bound as follows.
\begin{Theorem}
\label{eq:MinimaxOptimalboundOfmBesov}
Assume that $0 < p,q \leq \infty$ and $s > (1/p - 1/2)_+$ and $P_X$ is the uniform distribution over $[0,1]^d$.
Regarding $d$ as a constant, the minimax learning rate in the asymptotics of $n \to \infty$ is lower bounded as follows:
There exists a constant $\widehat{C}_1$ such that
\begin{align}
\inf_{\fhat} \sup_{\ftrue \in U(MB_{p,q}^s)}\EE_{D_n}[\|\fhat - \ftrue \|_{\LPiPx}^2] 
\geq \widehat{C}_1 n^{- \frac{2s}{2s + 1}} \log(n)^{\frac{2(d-1)(s + 1/2 - 1/q)_+}{2s+1}}
\label{eq:minimaxLp}
\end{align}
where ``inf'' is taken over all measurable functions of the observations $(x_i,y_i)_{i=1}^n$ and the expectation is taken for the sample distribution. 
\end{Theorem}

The proof is given in Appendix \ref{sec:MinimaxMixedSmooth}.
Because of this theorem, our bound given in Theorem \ref{thm:EstimationErrorNN} achieves the minimax optimal rate 
in the regime of $p < 2$ and $1/2 - 1/q > 0$ up to $\log(n)^2$ order.
Even outside of this parameter setting, the discrepancy between our upper bound 
and the minimax lower bound is just a poly-$\log$ oder.  
See also \cite{neumann2012multivariate} for some other related spaces and specific examples such as $p=q=2$. 

\section{Conclusion}
This paper investigated the learning ability of deep ReLU neural network 
when the target function is in a Besov space or a mixed smooth Besov space.
Based on the analysis for the Besov space, it is shown that 
deep learning using the ReLU activation can achieve the minimax optimal rate 
and outperform the linear method when $p < 2$ which indicates the spatial inhomogeneity of the shape of the target function.
The analysis for the mixed smooth Besov space shows that 
deep learning can adaptively avoid the curse of dimensionality. 
The bound is derived by sparse grid technique.
All analyses in the paper adopted the cardinal B-spline expansion and 
the adaptive non-linear approximation technique, which allowed us to show the minimax optimal rate.
The consequences of the analyses partly support the superiority of deep leaning 
in terms of adaptivity and ability to avoid curse of dimensionality.
From more high level view point, these favorable property is reduced to 
its high feature extraction ability.

This paper did not discuss any optimization aspect of deep learning.
However, it is important to investigate what kind of practical algorithms can actually achieve the optimal rate derived in this paper
in an efficient way.
We leave this important issue for future work.

\section*{Acknowledgment}
TS was partially supported by MEXT Kakenhi (25730013, 25120012, 26280009, 15H05707 and
18H03201), Japan Digital Design, and JST-CREST.

\bibliographystyle{iclr2019_conference}
\bibliography{main,main_colt}

\begin{thebibliography}{66}
\providecommand{\natexlab}[1]{#1}
\providecommand{\url}[1]{\texttt{#1}}
\expandafter\ifx\csname urlstyle\endcsname\relax
  \providecommand{\doi}[1]{doi: #1}\else
  \providecommand{\doi}{doi: \begingroup \urlstyle{rm}\Url}\fi

\bibitem[Adams \& Fournier(2003)Adams and Fournier]{adams2003sobolev}
R.A. Adams and J.J.F. Fournier.
\newblock \emph{Sobolev Spaces}.
\newblock Pure and Applied Mathematics. Elsevier Science, 2003.

\bibitem[Barron(1991)]{barron99l99lb}
Andrew Barron.
\newblock Approximation and estimation bounds for artificial neural networks.
\newblock In \emph{Proceedings of the Fourth Annual Workshop on Computational
  Learning Theory}, pp.\  243--249, 1991.

\bibitem[Barron(1993)]{barron1993universal}
Andrew~R Barron.
\newblock Universal approximation bounds for superpositions of a sigmoidal
  function.
\newblock \emph{IEEE Transactions on Information theory}, 39\penalty0
  (3):\penalty0 930--945, 1993.

\bibitem[Barron(1994)]{barron1994approximation}
Andrew~R Barron.
\newblock Approximation and estimation bounds for artificial neural networks.
\newblock \emph{Machine Learning}, 14\penalty0 (1):\penalty0 115--133, 1994.

\bibitem[Bianchini \& Scarselli(2014)Bianchini and
  Scarselli]{bianchini2014complexity}
Monica Bianchini and Franco Scarselli.
\newblock On the complexity of neural network classifiers: A comparison between
  shallow and deep architectures.
\newblock \emph{IEEE transactions on neural networks and learning systems},
  25\penalty0 (8):\penalty0 1553--1565, 2014.

\bibitem[B{\"o}lcskei et~al.(2017)B{\"o}lcskei, Grohs, Kutyniok, and
  Petersen]{bolcskei2017optimal}
Helmut B{\"o}lcskei, Philipp Grohs, Gitta Kutyniok, and Philipp Petersen.
\newblock Optimal approximation with sparsely connected deep neural networks.
\newblock \emph{arXiv preprint arXiv:1705.01714}, 2017.

\bibitem[Chui et~al.(1994)Chui, Li, and Mhaskar]{chui1994neural}
CK~Chui, Xin Li, and HN~Mhaskar.
\newblock Neural networks for localized approximation.
\newblock \emph{Mathematics of Computation}, 63\penalty0 (208):\penalty0
  607--623, 1994.

\bibitem[Cohen et~al.(2001)Cohen, Dahmen, Daubechies, and
  DeVore]{cohen2001tree}
Albert Cohen, Wolfgang Dahmen, Ingrid Daubechies, and Ronald DeVore.
\newblock Tree approximation and optimal encoding.
\newblock \emph{Applied and Computational Harmonic Analysis}, 11\penalty0
  (2):\penalty0 192--226, 2001.

\bibitem[Cohen \& Shashua(2016)Cohen and Shashua]{ICML:Cohen+Shashua:2016}
Nadav Cohen and Amnon Shashua.
\newblock Convolutional rectifier networks as generalized tensor
  decompositions.
\newblock In \emph{Proceedings of the 33th International Conference on Machine
  Learning}, volume~48 of \emph{JMLR Workshop and Conference Proceedings}, pp.\
   955--963, 2016.

\bibitem[Cohen et~al.(2016)Cohen, Sharir, and Shashua]{cohen2016expressive}
Nadav Cohen, Or~Sharir, and Amnon Shashua.
\newblock On the expressive power of deep learning: A tensor analysis.
\newblock In \emph{The 29th Annual Conference on Learning Theory}, pp.\
  698--728, 2016.

\bibitem[Cybenko(1989)]{cybenko1989approximation}
George Cybenko.
\newblock Approximation by superpositions of a sigmoidal function.
\newblock \emph{Mathematics of Control, Signals, and Systems (MCSS)},
  2\penalty0 (4):\penalty0 303--314, 1989.

\bibitem[DeVore \& Popov(1988)DeVore and Popov]{devore1988interpolation}
Ronald~A DeVore and Vasil~A Popov.
\newblock Interpolation of besov spaces.
\newblock \emph{Transactions of the American Mathematical Society},
  305\penalty0 (1):\penalty0 397--414, 1988.

\bibitem[DeVore et~al.(1993)DeVore, Kyriazis, Leviatan, and
  Tikhomirov]{devore1993wavelet}
Ronald~A DeVore, George Kyriazis, Dany Leviatan, and Vladimir~M Tikhomirov.
\newblock Wavelet compression and nonlinearn-widths.
\newblock \emph{Advances in Computational Mathematics}, 1\penalty0
  (2):\penalty0 197--214, 1993.

\bibitem[Donoho et~al.(1998)Donoho, Johnstone, et~al.]{donoho1998minimax}
David~L Donoho, Iain~M Johnstone, et~al.
\newblock Minimax estimation via wavelet shrinkage.
\newblock \emph{The Annals of Statistics}, 26\penalty0 (3):\penalty0 879--921,
  1998.

\bibitem[D{\~u}ng(1990)]{dung1990recovery}
Dinh D{\~u}ng.
\newblock On recovery and one-sided approximation of periodic functions of
  several variables.
\newblock In \emph{Dokl. Akad. SSSR}, volume 313, pp.\  787--790, 1990.

\bibitem[D{\~u}ng(1991)]{ICM-Satellite:DinhDung:1991}
Dinh D{\~u}ng.
\newblock On optimal recovery of multivariate periodic functions.
\newblock In Satoru Igari (ed.), \emph{ICM-90 Satellite Conference
  Proceedings}, pp.\  96--105, Tokyo, 1991. Springer Japan.
\newblock ISBN 978-4-431-68168-7.

\bibitem[D{\~u}ng(1992)]{dung1992optimal}
Dinh D{\~u}ng.
\newblock Optimal recovery of functions of a certain mixed smoothness.
\newblock \emph{Vietnam Journal of Mathematics}, 20\penalty0 (2):\penalty0
  18--32, 1992.

\bibitem[D{\~u}ng(2011{\natexlab{a}})]{Complexity:Dung:2011}
Dinh D{\~u}ng.
\newblock B-spline quasi-interpolant representations and sampling recovery of
  functions with mixed smoothness.
\newblock \emph{Journal of Complexity}, 27\penalty0 (6):\penalty0 541--567,
  2011{\natexlab{a}}.

\bibitem[D{\~u}ng(2011{\natexlab{b}})]{dung2011optimal}
Dinh D{\~u}ng.
\newblock Optimal adaptive sampling recovery.
\newblock \emph{Advances in Computational Mathematics}, 34\penalty0
  (1):\penalty0 1--41, 2011{\natexlab{b}}.

\bibitem[D{\~u}ng et~al.(2016)D{\~u}ng, Temlyakov, and
  Ullrich]{dung2016hyperbolic}
Dinh D{\~u}ng, Vladimir~N Temlyakov, and Tino Ullrich.
\newblock Hyperbolic cross approximation.
\newblock \emph{arXiv preprint arXiv:1601.03978}, 2016.

\bibitem[Galeev(1996)]{Galeev1996}
\'E.~M. Galeev.
\newblock Linear widths of h\"older-nikol'skii classes of periodic functions of
  several variables.
\newblock \emph{Matematicheskie Zametki,}, 59\penalty0 (2):\penalty0 189--199,
  1996.

\bibitem[Gin{\'e} \& Nickl(2015)Gin{\'e} and Nickl]{GineNickl2015mathematical}
E.~Gin{\'e} and R.~Nickl.
\newblock \emph{Mathematical Foundations of Infinite-Dimensional Statistical
  Models}.
\newblock Cambridge Series in Statistical and Probabilistic Mathematics.
  Cambridge University Press, 2015.

\bibitem[Glorot et~al.(2011)Glorot, Bordes, and Bengio]{glorot2011deep}
Xavier Glorot, Antoine Bordes, and Yoshua Bengio.
\newblock Deep sparse rectifier neural networks.
\newblock In \emph{Proceedings of the 14th International Conference on
  Artificial Intelligence and Statistics}, volume~15 of \emph{Proceedings of
  Machine Learning Research}, pp.\  315--323, 2011.

\bibitem[Hornik(1991)]{hornik1991approximation}
Kurt Hornik.
\newblock Approximation capabilities of multilayer feedforward networks.
\newblock \emph{Neural Networks}, 4\penalty0 (2):\penalty0 251--257, 1991.

\bibitem[Imaizumi \& Fukumizu(2018)Imaizumi and
  Fukumizu]{arXiv:Imaizumi+Fukumizu:2018}
Masaaki Imaizumi and Kenji Fukumizu.
\newblock Deep neural networks learn non-smooth functions effectively.
\newblock \emph{arXiv preprint arXiv:1802.04474}, 2018.

\bibitem[Kanagawa et~al.(2016)Kanagawa, Suzuki, Kobayashi, Shimizu, and
  Tagami]{ICML:Kanawaga+etal:2016}
Heishiro Kanagawa, Taiji Suzuki, Hayato Kobayashi, Nobuyuki Shimizu, and
  Yukihiro Tagami.
\newblock Gaussian process nonparametric tensor estimator and its minimax
  optimality.
\newblock In \emph{Proceedings of the 33rd International Conference on Machine
  Learning ({ICML}2016)}, pp.\  1632--1641, 2016.

\bibitem[Klusowski \& Barron(2016)Klusowski and Barron]{klusowski2016risk}
Jason~M Klusowski and Andrew~R Barron.
\newblock Risk bounds for high-dimensional ridge function combinations
  including neural networks.
\newblock \emph{arXiv preprint arXiv:1607.01434}, 2016.

\bibitem[Liang \& Srikant(2016)Liang and Srikant]{liang2016deep}
Shiyu Liang and R~Srikant.
\newblock Why deep neural networks for function approximation?
\newblock \emph{arXiv preprint arXiv:1610.04161}, 2016.
\newblock ICLR2017.

\bibitem[Mallat(1999)]{Mallat99a}
Stephane Mallat.
\newblock \emph{A Wavelet Tour of Signal Processing}.
\newblock Academic Press, 1999.

\bibitem[Meier et~al.(2009)Meier, {van de Geer}, and
  B{\"u}hlmann]{AS:Meier+Geer+Buhlmann:2009}
Lukas Meier, Sara {van de Geer}, and Peter B{\"u}hlmann.
\newblock High-dimensional additive modeling.
\newblock \emph{The Annals of Statistics}, 37\penalty0 (6B):\penalty0
  3779--3821, 2009.

\bibitem[Mhaskar(1996)]{mhaskar1996neural}
Hrushikesh~N Mhaskar.
\newblock Neural networks for optimal approximation of smooth and analytic
  functions.
\newblock \emph{Neural computation}, 8\penalty0 (1):\penalty0 164--177, 1996.

\bibitem[Mhaskar \& Micchelli(1992)Mhaskar and
  Micchelli]{mhaskar1992approximation}
Hrushikesh~N Mhaskar and Charles~A Micchelli.
\newblock Approximation by superposition of sigmoidal and radial basis
  functions.
\newblock \emph{Advances in Applied mathematics}, 13\penalty0 (3):\penalty0
  350--373, 1992.

\bibitem[Mhaskar(1993)]{mhaskar1993approximation}
Hrushikesh~Narhar Mhaskar.
\newblock Approximation properties of a multilayered feedforward artificial
  neural network.
\newblock \emph{Advances in Computational Mathematics}, 1\penalty0
  (1):\penalty0 61--80, 1993.

\bibitem[Montanelli \& Du(2017)Montanelli and Du]{MontanelliDu2017}
Hadrien Montanelli and Qiang Du.
\newblock Deep relu networks lessen the curse of dimensionality.
\newblock \emph{arXiv preprint arXiv:1712.08688}, 2017.

\bibitem[Montufar et~al.(2014)Montufar, Pascanu, Cho, and
  Bengio]{NIPS2014_5422}
Guido~F. Montufar, Razvan Pascanu, Kyunghyun Cho, and Yoshua Bengio.
\newblock On the number of linear regions of deep neural networks.
\newblock In Z.~Ghahramani, M.~Welling, C.~Cortes, N.d. Lawrence, and K.q.
  Weinberger (eds.), \emph{Advances in Neural Information Processing Systems
  27}, pp.\  2924--2932. Curran Associates, Inc., 2014.

\bibitem[Myronyuk(2016)]{myronyuk2016kolmogorov}
V~Myronyuk.
\newblock {K}olmogorov widths of the anisotropic {B}esov classes of periodic
  functions of many variables.
\newblock \emph{Ukrainian Mathematical Journal}, 68\penalty0 (5), 2016.

\bibitem[Nair \& Hinton(2010)Nair and Hinton]{nair2010rectified}
Vinod Nair and Geoffrey~E Hinton.
\newblock Rectified linear units improve restricted boltzmann machines.
\newblock In \emph{Proceedings of the 27th International Conference on Machine
  Learning}, pp.\  807--814, 2010.

\bibitem[Neumann(2000)]{neumann2012multivariate}
Michael~H. Neumann.
\newblock Multivariate wavelet thresholding in anisotropic function spaces.
\newblock \emph{Statistica Sinica}, 10\penalty0 (2):\penalty0 399--431, 2000.

\bibitem[Peetre \& Dept(1976)Peetre and Dept]{peetre1976new}
J.~Peetre and Duke University.~Mathematics Dept.
\newblock \emph{New thoughts on Besov spaces}.
\newblock Duke University mathematics series. Mathematics Dept., Duke
  University, 1976.

\bibitem[Petersen \& Voigtlaender(2017)Petersen and
  Voigtlaender]{petersen2017optimal}
Philipp Petersen and Felix Voigtlaender.
\newblock Optimal approximation of piecewise smooth functions using deep relu
  neural networks.
\newblock \emph{arXiv preprint arXiv:1709.05289}, 2017.

\bibitem[Petrushev(1998)]{petrushev1998approximation}
Pencho~P Petrushev.
\newblock Approximation by ridge functions and neural networks.
\newblock \emph{SIAM Journal on Mathematical Analysis}, 30\penalty0
  (1):\penalty0 155--189, 1998.

\bibitem[Pinkus(1999)]{pinkus1999approximation}
Allan Pinkus.
\newblock Approximation theory of the mlp model in neural networks.
\newblock \emph{Acta numerica}, 8:\penalty0 143--195, 1999.

\bibitem[Poole et~al.(2016)Poole, Lahiri, Raghu, Sohl-Dickstein, and
  Ganguli]{NIPS:Poole+etal:2016}
Ben Poole, Subhaneil Lahiri, Maithreyi Raghu, Jascha Sohl-Dickstein, and Surya
  Ganguli.
\newblock Exponential expressivity in deep neural networks through transient
  chaos.
\newblock In D.~D. Lee, M.~Sugiyama, U.~V. Luxburg, I.~Guyon, and R.~Garnett
  (eds.), \emph{Advances in Neural Information Processing Systems 29}, pp.\
  3360--3368. Curran Associates, Inc., 2016.

\bibitem[Raskutti et~al.(2012{\natexlab{a}})Raskutti, Wainwright, and
  Yu]{JMLR:Raskutti+Martin:2012}
Garvesh Raskutti, Martin Wainwright, and Bin Yu.
\newblock Minimax-optimal rates for sparse additive models over kernel classes
  via convex programming.
\newblock \emph{Journal of Machine Learning Research}, 13:\penalty0 389--427,
  2012{\natexlab{a}}.

\bibitem[Raskutti et~al.(2012{\natexlab{b}})Raskutti, Wainwright, and
  Yu]{raskutti2012minimax}
Garvesh Raskutti, Martin~J Wainwright, and Bin Yu.
\newblock Minimax-optimal rates for sparse additive models over kernel classes
  via convex programming.
\newblock \emph{The Journal of Machine Learning Research}, 13\penalty0
  (1):\penalty0 389--427, 2012{\natexlab{b}}.

\bibitem[Romanyuk(2001)]{Romanyuk2001}
A.~S. Romanyuk.
\newblock Linear widths of the besov classes of periodic functions of many
  variables. ii.
\newblock \emph{Ukrainian Mathematical Journal}, 53\penalty0 (6):\penalty0
  965--977, Jun 2001.

\bibitem[Romanyuk(2009)]{Romanyuk2009Kolmogorovwidth}
A.~S. Romanyuk.
\newblock Bilinear approximations and {K}olmogorov widths of periodic {B}esov
  classes.
\newblock \emph{Theory of Operators, Differential Equations, and the Theory of
  Functions}, 6\penalty0 (1):\penalty0 222--236, 2009.
\newblock Proc. of the Institute of Mathematics, Ukrainian National Academy of
  Sciences.

\bibitem[Schmeisser(1987)]{schmeisser1987unconditional}
H-J Schmeisser.
\newblock An unconditional basis in periodic spaces with dominating mixed
  smoothness properties.
\newblock \emph{Analysis Mathematica}, 13\penalty0 (2):\penalty0 153--168,
  1987.

\bibitem[{Schmidt-Hieber}(2017)]{2017arXiv170806633S}
J.~{Schmidt-Hieber}.
\newblock {Nonparametric regression using deep neural networks with ReLU
  activation function}.
\newblock \emph{ArXiv e-prints}, August 2017.

\bibitem[Sickel \& Ullrich(2009)Sickel and Ullrich]{sickel2009tensor}
Winfried Sickel and Tino Ullrich.
\newblock Tensor products of {S}obolev--{B}esov spaces and applications to
  approximation from the hyperbolic cross.
\newblock \emph{Journal of Approximation Theory}, 161\penalty0 (2):\penalty0
  748--786, 2009.

\bibitem[Sickel \& Ullrich(2011)Sickel and Ullrich]{sickel2011spline}
Winfried Sickel and Tino Ullrich.
\newblock Spline interpolation on sparse grids.
\newblock \emph{Applicable Analysis}, 90\penalty0 (3-4):\penalty0 337--383,
  2011.

\bibitem[Signoretto et~al.(2010)Signoretto, Lathauwer, and
  Suykens]{TechRepo:Signoretto+etal:2010}
M.~Signoretto, L.~De Lathauwer, and J.A.K. Suykens.
\newblock Nuclear norms for tensors and their use for convex multilinear
  estimation.
\newblock Technical Report 10-186, ESAT-SISTA, K.U.Leuven, 2010.

\bibitem[Smolyak(1963)]{smolyak1963quadrature}
Sergey Smolyak.
\newblock Quadrature and interpolation formulas for tensor products of certain
  classes of functions.
\newblock In \emph{Soviet Math. Dokl.}, volume~4, pp.\  240--243, 1963.

\bibitem[Sonoda \& Murata(2015)Sonoda and Murata]{sonoda2015neural}
Sho Sonoda and Noboru Murata.
\newblock Neural network with unbounded activation functions is universal
  approximator.
\newblock \emph{Applied and Computational Harmonic Analysis}, 2015.

\bibitem[Suzuki et~al.(2016)Suzuki, Kanagawa, Kobayashi, Shimizu, and
  Tagami]{suzuki2016minimax}
Taiji Suzuki, Heishiro Kanagawa, Hayato Kobayashi, Nobuyuki Shimizu, and
  Yukihiro Tagami.
\newblock Minimax optimal alternating minimization for kernel nonparametric
  tensor learning.
\newblock In \emph{Advances In Neural Information Processing Systems}, pp.\
  3783--3791, 2016.

\bibitem[Temlyakov(1982)]{MathUSSRSob:Temlyakov:1982}
V.N. Temlyakov.
\newblock Approximation of periodic functions of several variables with bounded
  mixed difference.
\newblock \emph{Math. USSR Sb}, 41\penalty0 (1):\penalty0 53--66, 1982.

\bibitem[Temlyakov(1993{\natexlab{a}})]{Book:Temlyakov:1993}
V.N. Temlyakov.
\newblock \emph{Approximation of Periodic Functions}.
\newblock Nova Science Publishers, 1993{\natexlab{a}}.

\bibitem[Temlyakov(1993{\natexlab{b}})]{JComplexity:Temlyakov:1993}
V.N. Temlyakov.
\newblock On approximate recovery of functions with bounded mixed derivative.
\newblock \emph{Journal of Complexity}, 9:\penalty0 41--59, 1993{\natexlab{b}}.

\bibitem[Tikhomirov(1960)]{tikhomirov1960diameters}
Vladimir~Mikhailovich Tikhomirov.
\newblock Diameters of sets in function spaces and the theory of best
  approximations.
\newblock \emph{Uspekhi Matematicheskikh Nauk}, 15\penalty0 (3):\penalty0
  81--120, 1960.

\bibitem[Triebel(1983)]{triebel1983theory}
Hans Triebel.
\newblock \emph{Theory of function spaces}.
\newblock Monographs in mathematics. Birkh{\"a}user Verlag, 1983.
\newblock ISBN 9783764313814.

\bibitem[van~der Vaart \& Wellner(1996)van~der Vaart and
  Wellner]{Book:VanDerVaart:WeakConvergence}
A.~W. van~der Vaart and J.~A. Wellner.
\newblock \emph{Weak Convergence and Empirical Processes: With Applications to
  Statistics}.
\newblock Springer, New York, 1996.

\bibitem[Vyb{\'a}ral(2008)]{vybiral2008widths}
Jan Vyb{\'a}ral.
\newblock Widths of embeddings in function spaces.
\newblock \emph{Journal of Complexity}, 24:\penalty0 545--570, 2008.

\bibitem[Williamson \& Bartlett(1992)Williamson and
  Bartlett]{williamson1992splines}
Robert~C Williamson and Peter~L Bartlett.
\newblock Splines, rational functions and neural networks.
\newblock In \emph{Advances in Neural Information Processing Systems}, pp.\
  1040--1047, 1992.

\bibitem[Yang \& Barron(1999)Yang and Barron]{AS:Yang+Barron:99}
Yuhong Yang and Andrew Barron.
\newblock Information-theoretic determination of minimax rates of convergence.
\newblock \emph{The Annals of Statistics}, 27\penalty0 (5):\penalty0
  1564--1599, 1999.

\bibitem[Yarotsky(2016)]{DBLP:journals/corr/Yarotsky16}
Dmitry Yarotsky.
\newblock Error bounds for approximations with deep relu networks.
\newblock \emph{CoRR}, abs/1610.01145, 2016.

\bibitem[Zhang et~al.(2002)Zhang, Wong, and Zheng]{zhang2002wavelet}
Shuanglin Zhang, Man-Yu Wong, and Zhongguo Zheng.
\newblock Wavelet threshold estimation of a regression function with random
  design.
\newblock \emph{Journal of multivariate analysis}, 80\penalty0 (2):\penalty0
  256--284, 2002.

\end{thebibliography}

\appendix

\section{Proof of Lemma \ref{lemm:Mnapproximation}}
\label{sec:ProofBsplineApprox}
\begin{proof}[Proof of Lemma \ref{lemm:Mnapproximation}]
First note that $\calN_m(x) = \frac{1}{m!} \sum_{j=0}^{m+1} (-1)^j {m+1 \choose j} ( x-j)_+^m$ (see Eq. (4.28) of \cite{mhaskar1992approximation} for example).
Thus, if we can make an approximation of $\eta(x)^m$, then by taking a summation of those basis, we  obtain 
an approximate of $\calN_m(x)$.
It is shown by \cite{DBLP:journals/corr/Yarotsky16} that,
for $D \in \Natural$ and any $\epsilon > 0$,  there exists a neural network $\phi_{\mathrm{mult}} \in \Phi(L,W,S,B)$ with
$L = \lceil \log_2\left( \frac{3^D}{\epsilon}  \right)+5 \rceil \lceil\log_2(D)\rceil$, $W = 6d $, $S = LW^2$ and  $B = 1$  such that 
$$
\sup_{x \in [0,1]^d} \left | \phi_{\mathrm{mult}}(x_1,\dots,x_D)  - \prod_{i=1}^D x_ i \right | \leq \epsilon,
$$
and $\phi_{\mathrm{mult}}(0,\dots,0) = 0$
for $y \in \Real^d$ such that $\prod_{j=1}^d y_j=0$. 
Moreover, for any $M > 0$, we can realize the function 
$\min\{M,\max\{x,0\}\}$ by a single-layer neural network $\phi_{(0,M)}(x) := \eta(x) - \eta(x-M) (=\min\{M,\max\{x,0\}\})$.
Thus, for $x \in \Real$, it holds that
$$
\sup_{x \in [0,M]} \left|
\phi_{\mathrm{mult}}(\phi_{(0,1)}(x/M),\dots,\phi_{(0,1)}(x/M)) - (\phi_{(0,1)}(x/M))^m
\right| \leq \epsilon.
$$
Now, since $\calN_m(x) = 0$ for $x \not\in [0,m+1]$, it also holds
$
\calN_m(x) = \frac{1}{m!} \sum_{j=0}^{m+1} (-1)^j {m+1 \choose j}\phi_{(0,m+1 -j)}(x - j)^m
= \frac{1}{m!} \sum_{j=0}^{m+1} (-1)^j {m+1 \choose j} (m+1)^m \phi_{(0,1 -j/(m+1))}((x-j)/(m+1))^m.
$
Therefore, letting 
$$
f(x) = \frac{1}{m!} \sum_{j=0}^{m+1} (-1)^j (m+1)^m
{m+1 \choose j} 
\phi_{\mathrm{mult}}\Bigg(\underbrace{\phi_{(0,1 -\frac{j}{m+1})}\left(
\frac{x-j}{m+1}\right),\dots,\phi_{(0,1 -\frac{j}{m+1})}\left(\frac{x-j}{m+1}\right)}_{\text{$m$-times}}\Bigg),
$$ 
we have that $f(x) = 0$ for all $x \leq 0$ and 
\begin{align*}
& \sup_{0 \leq x \leq m+1} | \calN_m(x)  -
f(x) | \leq \frac{1}{m!} \sum_{j=0}^{m+1} {m+1 \choose j} (m+1)^m \epsilon
\leq \frac{(m+1)^m}{\sqrt{2\pi} m^{m+1/2} e^{-m}} 2^{m+1} \epsilon  \\
&~~~~~~~~~~~ \leq e \frac{ (2e)^m}{\sqrt{m}} \epsilon =: \epsilon',
\end{align*}
where we used $ \sum_{j=0}^{m+1} {m+1 \choose j}  = 2^{m+1}$
and Stirling's approximation $m! \geq \sqrt{2\pi} m^{m+1/2} e^{-m}$ in the second inequality.
Hence, we also have
\begin{align*}
& f(x) = 
 \frac{1}{m!} \sum_{j=0}^{m+1} (-1)^j  {m+1 \choose j} (m+1)^m  \\
& ~~~~~~~~~~~~~~~~~~~~~~~~\times \phi_{\mathrm{mult}}\left(\phi_{(0,1 -\frac{j}{m+1})}\left(\frac{m+1-j}{m+1}\right),\dots,
\phi_{(0,1 -\frac{j}{m+1})}\left(\frac{m+1-j}{m+1}\right)\right) \\
& ~~~~~~~~=: \delta'~~(\forall x > m+1).
\end{align*} 
It holds that $|\delta'| \leq \epsilon'$.
Because of this and noting $0 \leq \calN_m(x) \leq 1$, we see that $g(x) := \phi_{(0,1)}(f(x) - \frac{\delta'}{m+1} \phi_{(0,m+1)}(x)) $ 
yields
$$
\sup_{x \in \Real}|\calN_m(x)  - g(x)| \leq 2 \epsilon',
$$
$\sup_{x \in \Real} |g(x)| \leq 1$, and  $g(x) =0$ for all $x \not \in [0,m+1]$.
Hence, by applying $\phi_{\mathrm{mult}}$ again, we finally obtain that 
\begin{align*}
& \sup_{x \in [0,1]^d} |M_{0,0}^d(x) - \phi_{\mathrm{mult}}(g(x_1),\dots,g(x_d))|  \\
\leq & 
\sup_{x \in [0,1]^d} \left|M_{0,0}^d(x) - \prod_{j=1}^d g(x_j) \right| 
+ \sup_{x \in [0,1]^d} \left |\prod_{j=1}^d g(x_j) - \phi_{\mathrm{mult}}(g(x_1),\dots,g(x_d)) \right| \\
\leq & 
2 d \epsilon' + \epsilon.
\end{align*}
We again applying $\phi_{(0,1)}$, we obtain that 
$h = \phi_{(0,1)} \circ \phi_{\mathrm{mult}}(g(x_1),\dots,g(x_d))$ 
satisfies $\|M_{0,0}^d - h\|_{L^\infty(\Real^d)} \leq 2 d \epsilon' + \epsilon$,
$h(x) = 0$ for all $x \not \in [0,m+1]^d$, and $\|h\|_\infty \leq 1$.
Finally, by carefully checking the network construction, 
it is shown that 
$h \in 
\Phi(L,W,S,B)$ with
$L = 3 +  2 \lceil \log\left(  \frac{3^{d\vee m} }{\epsilon}  \right)+5 \rceil \lceil\log_2(d \vee m)\rceil$, 
$W = 6 d m(m+2) + 2 d$, $S = L W^2$ and  $B = 2 (m+1)^m$.
Hence, resetting $\epsilon \leftarrow 2d \epsilon' + \epsilon = (1 + 2d e \frac{(2e)^m}{\sqrt{m}}) \epsilon$, this $h$ is the desired $\check{M}$.
\end{proof}

\section{Proof of Proposition \ref{prop:BesovApproxByNN}}
\label{sec:ProofPropBesovApp}

For the order $m \in \Natural$ of the cardinal B-spline bases, let $J(k) = \{- m,  - m +1, \dots, 2^k-1, 2^k \}^d$ and 
the quasi-norm of the coefficient $(\alpha_{k,j})_{k,j}$ for $k \in \Natural_+$ and $j \in J(k)$ be 
$$
\| (\alpha_{k,j})_{k,j}\|_{b^s_{p,q}}. =  \left\{\sum_{k \in \Natural_+} \left[ 2^{k(s - d/p)}\Big(\sum_{j \in J(k)} |\alpha_{k,j}|^p\Big)^{1/p}  \right]^q\right\}^{1/q}.
$$

\begin{Lemma}
\label{lemm:BSplineInterpolation}
Under one of the conditions \eqref{eq:sprConditions} in Proposition \ref{prop:BesovApproxByNN}
and the condition $0 < s < \min( m , m -1 + 1/p)$ where $m \in \Natural$ is the order of the cardinal B-spline bases, 
for any $f \in B^s_{p,q}(\Omega)$, there exists $f_N$ 
that satisfies
\begin{align}
\label{eq:ffNoptimalAdaptiveApprox}
\|f - f_N\|_{L^r(\Omega)} \lesssim N^{-s /d} \|f\|_{B^s_{p,q}}
\end{align}
for $N \gg 1$, and has the following form:
\begin{align}
\label{eq:fNformat}
f_N(x) = \sum_{k=0}^K \sum_{j \in J(k)} \alpha_{k,j} M_{k,j}^d(x) + \sum_{k=K+1}^{K^*} \sum_{i=1}^{n_k}
\alpha_{k,j_i} M_{k,j_i}^d(x),
\end{align}
where 
$(j_i)_{i=1}^{n_k}\subset J(k)$,
$K = \lceil C_1 \log(N)/d \rceil$, 
$K^* = \lceil \log(\lambda N) \nu^{-1} \rceil + K + 1$,
$n_k = \lceil \lambda N 2^{-\nu (k - K)}\rceil~
(k=K+1,\dots,K^*)$
for $\delta = d(1/p - 1/r)_+$ and $\nu = (s - \delta)/(2\delta)$, 
 and the real number constants $C_1 > 0$ and $\lambda > 0$ 
are chosen to satisfy
$\sum_{k=1}^K (2^k+m )^d + \sum_{k=K+1}^{K^*} n_k \leq N$ independently to $N$.
Moreover, we can choose the coefficients $(\alpha_{k,j})$ to satisfy 
$$
\|(\alpha_{k,j})_{k,j}\|_{b^s_{p,q}} \lesssim \|f\|_{B_{p,q}^s}.
$$

\end{Lemma}

\begin{proof}[Proof of Lemma \ref{lemm:BSplineInterpolation}]

\cite{devore1988interpolation} 
constructed a linear bounded operator $P_k$ having the following form:
\begin{align}
P_k(f)(x) = \sum_{j \in J(k)} a_{k,j} M_{k,j}^d(x)
\label{eq:PkDecomposition}
\end{align}
where $\alpha_{k,j}$ is constructed in a certain way, where 
for every $f \in L^p([0,1]^d)$ with $0 < p \leq \infty$, it holds 
\begin{align}
\| f - P_k(f) \|_{L^p} \leq C w_{r,p}(f,2^{-k}).
\label{eq:fPkwkpIneq}
\end{align}
Let
$$
p_k(f) := P_k(f) - P_{k-1}(f),~~P_{-1}(f) = 0.
$$
Then, it is shown that for $0 < p, q \leq \infty$ and $0 < s < \min( m , m -1 + 1/p)$, 
$f$ belongs to $B_{p,q}^s$ if and only if $f$ can be decomposed into 
$$
f = \sum_{k=0}^\infty p_k(f),
$$
with the convergence condition $\|(p_k(f))_{k=0}^{\infty} \|_{b_p^s(L^p)} < \infty$;
in particular, $\|f\|_{B_{p,q}^s} \simeq \| (p_k(f))_{k=0}^{\infty} \|_{b_p^s(L^p)}
=: (\sum_{k \in \Natural_+} (2^{s k} \|p_k\|_{L^p} )^q)^{1/q}
$.
Here, each $p_k$ can be expressed as $p_k(x) = \sum_{j \in J(k)} \alpha_{k,j} M_{k,j}^d(x)$ 
for a coefficient $(\alpha_{k,j})_{k,j}$ (which could be different from  $(a_{k,j})_{k,j}$ appearing in \Eqref{eq:PkDecomposition}).
Hence, $f \in B^s_{p,q}$ can be decomposed into 
\begin{align}
\label{eq:DevorePopovExpansion}
f = \sum_{k=0}^\infty \sum_{j \in J(k)}  \alpha_{k,j} M_{k,j}^d(x)
\end{align}
with convergence in the sence of $L^p$.
Moreover, it is shown that $\|p_k\|_{L^p} \simeq( 2^{-kd}\sum_{j \in J(k)} |\alpha_{k,j}|^p  )^{1/p}$ and thus
\begin{align}
\|f\|_{B_{p,q}^s} \simeq \| (\alpha_{k,j})_{k,j}\|_{b^s_{p,q}}.
\label{eq:BpqEquivalence}
\end{align}

Based on this decomposition, \cite{dung2011optimal} proposed an optimal adaptive recovery method 
such that 
the approximator has the form \eqref{eq:fNformat} under the conditions for $K,K^*,n_k$ given in the statement 
and satisfies the approximation accuracy \eqref{eq:ffNoptimalAdaptiveApprox}.
This can be proven by applying the proof of Theorem 3.1 in \cite{dung2011optimal} to the decomposition 
\eqref{eq:DevorePopovExpansion} instead of Eq. (3.8) of that paper. See also Theorem 5.4 of \cite{dung2011optimal}.
Moreover, the equivalence \eqref{eq:BpqEquivalence} gives the norm bound of the coefficient $(\alpha_{k,j})$.
\end{proof}

\begin{proof}[Proof of Proposition \ref{prop:BesovApproxByNN}]
Basically, we combine Lemma \ref{lemm:Mnapproximation}
and 
Lemma \ref{lemm:BSplineInterpolation}.
We substitute the approximated cardinal B-spline basis $\check{M}$ into 
the decomposition of $f_N$ \eqref{eq:fNformat}.
Let the set of indexes $(k,j) \in \Natural \times \Natural$ 
that  consists $f_N$ given in \Eqref{eq:fNformat}  
be $E_{N}$: $f_N = \sum_{(k,j)\in E_N} \alpha_{k,j} M_{k,j}^d$.
Accordingly, we set $\fcheck := \sum_{(k,j)\in E_N} \alpha_{k,j} \check{M}_{k,j}^d$.
For each $x \in \Real^d$, it holds that
\begin{align*}
|f_N(x) - \fcheck(x)|
& \leq \sum_{(k,j) \in E_N} |\alpha_{k,j} |  |M_{k,j}^d(x) - \check{M}_{k,j}^d(x)| \\
& \leq \epsilon \sum_{(k,j) \in E_N} |\alpha_{k,j} |   \boldone\{M_{k,j}^d(x) \neq 0\} \\
& \leq \epsilon (m+1)^d 
(1+K^*) 2^{ K^* (d/p-s)_+} \|f\|_{B^s_{p,q}} \\
& \lesssim \log(N) 
N^{(\nu^{-1} + d^{-1})(d/p-s)_+} \epsilon \|f\|_{B^s_{p,q}}, 
\end{align*}
where we used the definition of $K^*$ in the last inequality.
Therefore, for each $f \in U(B^s_{p,q}([0,1]^d))$, it holds that 
$$
\|f - \fcheck\|_{L^r}
\lesssim  
\|f - f_N\|_{L^r}
+
\|f_N- \fcheck\|_{L^r}
\lesssim \log(N) N^{(\nu^{-1} + d^{-1})(d/p-s)_+}\|f\|_{B^s_{p,q}} \epsilon + N^{-s/d}.
$$
By taking $\epsilon$ to satisfy 
$\log(N) N^{(\nu^{-1} + d^{-1})(d/p-s)_+}\epsilon \leq N^{-s/d}$
(i.e., $\epsilon \leq N^{-s/d - (\nu^{-1} + d^{-1})(d/p -s)_+}\log(N)^{-1}$), 
then we obtain the approximation error bound.

Next, we bound the magnitude of the coefficients. 
Each coefficient $\alpha_{j,k}$ satisfies $|\alpha_{j,k}| \lesssim 2^{k(d/p - s)_+}\|f\|_{B^s_{p,q}} \leq 2^{k(d/p - s)_+}
\lesssim N^{(\nu^{-1} + d^{-1})(d/p - s)_+}$ for $k \leq K^*$.
Finally, the magnitudes of the coefficients hidden in $\check{M}_{k,j}^d$ are evaluated.
Remembering that $\check{M}_{k,j}^m(x) = \check{M}(2^{k} x_1 - j_1,\dots, 2^{k} x_d - j_d)$,
we see that we just need to bound the quantity $2^{k}~(k \leq K^*)$. However, this is 
bounded by $2^k \leq N^{\nu^{-1} + d^{-1}}$ for $k \leq K^*$. Hence, we obtain the assertion.
\end{proof}

\section{Proof of Theorem \ref{eq:mBesovApproxByNN}}
\label{sec:ProofmBesovApproxNN}

Let $\Natural_+^d(e) := \{ s \in \Natural_+^d \mid s_i = 0, i \not \in e \}$
and
for 
$
k \in \Natural_+^d(e),
$
we define 
$2^{-k}:=(2^{-k_{i_1}}, \dots, 2^{-k_{i_{|e|}}}) \in \Real_+^{|e|}$ where $(i_1,\dots,i_{|e|}) = e$.
By defining  
$\|(g_k)_k\|_{b_{q}^{\alpha,e}} := \left(\sum_{k \in \Natural_+^d(e)} 
(2 ^{\alpha \|k\|_1} |g_k|)^q
\right)^{1/q}
$
for a sequence $(g_k)_{k \in \Natural_+^d(e)}$, then 
it holds that 
$$
|f|_{MB_{p,q}^{\alpha,e}} = \sum_{e \subset \{1,\dots,d\}} \|(w_{r,p}^e(f,2^{-k}))_k\|_{b_{q}^{\alpha,e}}.
$$

\begin{proof}[Proof of Theorem \ref{eq:mBesovApproxByNN}]

The result is immediately follows from Theorem \ref{eq:mBesoApproxSpline}.
Let the set of indexes of $(k,j)$ consisting of $R_K$ be $E_K$: $R_K(f) = \sum_{(k,j) \in E_K} \alpha_{k,j} M_{k,j}^d(x)$.
As in the proof of Proposition \ref{prop:BesovApproxByNN}, we approximate $R_K(f)$ by a neural network given as 
$$
\fcheck(x)= \sum_{(k,j) \in E_K} \alpha_{k,j} \check{M}_{k,j}^d(x).
$$
Each coefficient $\alpha_{j,k}$ satisfies $|\alpha_{j,k}| \lesssim 2^{\|k\|_1(1/p - s)_+}\|f\|_{MB^s_{p,q}} \lesssim 
2^{K^* (1/p - s)_+}.$
The difference between 
\begin{align*}
|R_K(f) - \fcheck(x)| 
& \leq \sum_{(k,j) \in E_K} |\alpha_{k,j}| |M_{k,j}^d(x) - \check{M}_{k,j}^d(x)| \\
& \leq \epsilon \sum_{(k,j) \in E_K} |\alpha_{k,j}| \boldone\{M_{k,j}^d(x) \neq 0 \} \\
& \lesssim \epsilon (m+1)^d (1+K^*)D_{K^*,d} 2^{K^*(1/p -s )_+}
\|f\|_{MB_{p,q}^s}.
\end{align*}
Therefore, by taking $\epsilon$ so that 
$\epsilon (m+1)^d (1+K^*)D_{K^*,d} 2^{K^*(1/p - s)_+} \leq 2^{-K s}$ is satisfied,
it holds that 
$$
|R_K(f) - \fcheck(x)|  \lesssim 2^{-K s}.
$$
By the inequality $D_{K^*,d} \leq e^{K^* + d-1}$, it suffices to let $\epsilon \leq \frac{e^{- K^*(s + (1/p -s)_+ + 1)}}{
[e(m+1)]^d (1+K^*)
}$.
The cardinality of $E(K)$  is bounded as 
\begin{align*}
& \sum_{\kappa=0,\dots,K} 2^\kappa {\kappa + d-1 \choose d-1} + \sum_{k: K < \|k\|_1 \leq K^*} n_k  \\
\leq & 2^{K+1} {K + d-1 \choose d-1} + \sum_{K < \kappa \leq K^*}  2^{K - \frac{s - \delta}{2 \delta}(\kappa -K)} 
{\kappa + d-1 \choose d-1} \\
\leq &
2^{K+1} D_{K,d} +  2^K (1 - 2^{- \frac{s - \delta}{2 \delta}})^{-1} D_{K^*,d} 
\leq (2 + (1 - 2^{- \frac{s - \delta}{2 \delta}})^{-1} ) 2^K D_{K^*,d} = N. 
\end{align*}
Since each unit $\check{M}_{k,j}^d$ requires width $W_0$, the whole width becomes $W = N W_0$.
The number of nonzero parameters to construct $\check{M}_{k,j}^d$ is bounded by $S = (L-1)W_0^2 N + N$.

Finally, the magnitudes of the coefficients hidden in $\check{M}_{k,j}^d$ are evaluated.
Remembering that $\check{M}_{k,j}^d(x) = \check{M}(2^{k_1} x_1 - j_1,\dots, 2^{k_d} x_d - j_d)$,
here maximum of $2^{k_j}$ is bounded by $2^{K^*} \lesssim N^{(1 + 1/\nu)}$. Hence, we obtain the assertion.
Similarly, it holds that $|\alpha_{j,k}| \lesssim N^{(1 + 1/\nu) \{1 \vee (1/p -s)_+\}}$.
\end{proof}

\section{Proof of Theorem \ref{eq:mBesoApproxSpline}}
\subsection{Preparation: sparse grid}
\label{sec:SparseGridBound}

Here, we give technical details behind the approximation bound.
The analysis utilizes the so called {\it sparse grid} technique \cite{smolyak1963quadrature}
which has been developed in the function approximation theory field. 

As we have seen in the above, in a typical B-spline approximation scheme, 
we put the basis functions $M_{k,j}^m(x)$ on a ``regular grid''
for $k=1,\dots,K$ and $(j_1,\dots,j_d) \in J(k)$, and 
take its superposition as $f(x) \approx \sum_{k =1,\dots,K} \sum_{j \in J(k)} \alpha_{k,j} M_{k,j}^m(x)$,
which consists of $O(2^{Kd})$ terms (see \Eqref{eq:fNformat}).
Hence, the number of parameters $O(2^{Kd})$ is affected by the dimensionality $d$ in an exponential order.
However, to approximate functions with mixed smoothness, we do not 
need to put the basis on the whole range of the regular grid.
Instead, we just need to put them on a {\it sparse grid} which is a subset of the regular grid and 
has much smaller cardinality than the whole set.
The approximation algorithm utilizing sparse grid is based on Smolyak's construction \citep{smolyak1963quadrature}
and its applications to mixed smooth spaces \citep{dung1990recovery,ICM-Satellite:DinhDung:1991,dung1992optimal,MathUSSRSob:Temlyakov:1982,Book:Temlyakov:1993,JComplexity:Temlyakov:1993}.
\cite{Complexity:Dung:2011} studied an optimal non-adaptive linear sampling recovery method for the mixed smooth Besov space
based on the cardinal B-spline bases.  
We adopt this method, and combining this with the adaptive technique developed in \cite{dung2011optimal},
we give the following approximation bound using a non-linear adaptive method to obtain better convergence for the setting $p < r$.


Before we state the theorem, we define an quasi-norm of a set of coefficients $\alpha_{k,j} \in \Real$ for $k \in \Natural_+^d$ and $j \in J_{m}^d(k)
:=
\{- m,  - m +1, \dots, 2^{k_1}-1, 2^{k_1} \}
\times \dots \times 
\{- m,  - m +1, \dots, 2^{k_d}-1, 2^{k_d} \}
$ as 
$$
\|(\alpha_{k,j})_{k,j}\|_{mb_{p,q}^{\alpha}}  
:= \left(\sum_{k \in \Natural_+^d} \left[2^{(\alpha - 1/p)\|k\|_1} \Big( \sum_{j \in J_m^d(k)} |\alpha_{k,j}|^p \Big)^{1/p} \right]^q\right)^{1/q}.
$$

\begin{Theorem}
\label{eq:mBesoApproxSpline}
Suppose that $0 < p,q,r \leq \infty$ and $\alpha > (1/p - 1/r)_+$.
Assume that the order $m \in \Natural$ of the cardinal B-spline satisfies $0 < s < \min( m , m -1 + 1/p)$.
Let $\delta = (1/p - 1/r)_+$.
Then, 
for any $f \in MB_{p,q}^s(\Omega)$ and $K > 0$,
there exists $R_K(f)$
such that $R_K(f)$ can be represented as 
$$
R_K(f)(x) = \sum_{\substack{k \in \Natural_+^d: \\ \|k\|_1 \leq K}} \sum_{j \in J_m^d(k)} \alpha_{k,j} M_{k,j}^d(x)
+ \sum_{\substack{k \in \Natural_+^d: \\ K < \|k\|_1 \leq K^*}} \sum_{i=1}^{n_k} \alpha_{k,j_i^{(k)}} M_{k,j_i^{(k)}}^d(x),
$$
where $K^* = \lceil K(1 + \frac{2\delta}{\alpha - \delta})  \rceil $,
$(j_i^{(k)})_{i=1}^{n_k} \subset J_m^d(k)$, and $n_k = \lceil 2^{K -\frac{\alpha - \delta}{2 \delta} (\|k\|_1 - K)} \rceil$,
and has the following properties:
\begin{itemize}
\item[(i)] For $p \geq r$, 
\begin{align*} 
\|f - R_K(f)\|_r \lesssim
2^{-K \alpha } D_{K,d}^{(1/\min(r,1) - 1/q)_+} \|f\|_{MB_{p,q}^s}.
\end{align*}
\item[(ii)] For $p < r$, 
\begin{align*}
\|f - R_K(f)\|_r \lesssim 
\begin{cases}
2^{-K \alpha } D_{K,d}^{(1/r - 1/q)_+}  \|f\|_{MB_{p,q}^s}& (r < \infty), \\
2^{-K \alpha } D_{K,d}^{(1 - 1/q)_+}  \|f\|_{MB_{p,q}^s}& (r = \infty).
\end{cases}
\end{align*}
\end{itemize}
Moreover, the coefficients $(\alpha_{k,j})_{k,j}$ can be taken to hold 
$\|(\alpha_{k,j})_{k,j}\|_{mb_{p,q}^{\alpha}}  \lesssim \|f\|_{MB_{p,q}^\alpha}$.
\end{Theorem}

The proof is given in Appendix \ref{sec:ProofOfTheoremmBesoApproxSpline}.
The total number of cardinal B-spline bases consisting of $R_K(f)$ can be evaluated as 
\begin{align*}
& 2^{K+1} {K + d-1 \choose d-1} + \sum_{k: K < \|k\|_1 \leq K^*} n_k  \\
\lesssim &
2^K D_{K,d} + 2^K D_{K^*,d}  \lesssim 2^K D_{K,d}~~~~~~~~(\because \text{\Eqref{eq:Dkd_upperbound}}).
\end{align*}
Here, $D_{K,d}$ can be evaluated as 
$$
D_{K,d} \lesssim 
K^{d-1} ~~~\text{or}~~~D_{K,d} \lesssim d^K.
$$
Therefore, the total number of bases can be evaluated as 
$$
2^K \min\{K^{d-1}, d^K\}
$$
which is much smaller than $2^{Kd}$ which is required to approximate functions in the ordinal Besov space (see Lemma \ref{lemm:BSplineInterpolation}).
In this proposition, $K$ controls the resolution and as $M$ goes to infinity, the approximation error goes to 0
exponentially fast.
A remarkable point in the proposition is in the construction of $R_K(f)$ in which 
the superposition is taken over $\|k\|_1 \leq M$ instead of $\|k\|_\infty \leq K^* = O(K)$.
Hence, the number of terms appearing in the summation is 
at most $O(2^K K^{d-1})$ while the full grid takes $O(2^{K d})$ terms.
This represents how the mixed smoothness is important to ease the curse of dimensionality.


Several aspects of the m-Besov space such as the optimal $N$-term approximation error 
and Kolmogorov widths have been extensively studied in the literature (see a comprehensive survey \cite{dung2016hyperbolic}).
An analogous result is already given by \cite{Complexity:Dung:2011} in which $\alpha > 1/p$ is assumed
and a linear interpolation method is investigated. 
However, our result only requires $\alpha > (1/p - 1/q)_+$.
This difference comes from a point that our analysis allows nonlinear adaptive interpolation instead of 
(linear) non-adaptive sampling considered in \cite{Complexity:Dung:2011}. 
Because of this, our bound is better than the optimal rate of linear methods 
\citep{Galeev1996,Romanyuk2001} and non-adaptive methods \citep{dung1990recovery,ICM-Satellite:DinhDung:1991,dung1992optimal,MathUSSRSob:Temlyakov:1982,Book:Temlyakov:1993,JComplexity:Temlyakov:1993}
especially in the regime of $p < r$
(\cite{dung1992optimal} also deals with adaptive method but does not cover $p < r$ for adaptive method). 
See Proposition \ref{eq:LinearWidthmBesov} for comparison.


%
%



\subsection{Proof of Theorem \ref{eq:mBesoApproxSpline}}
\label{sec:ProofOfTheoremmBesoApproxSpline}

Now we are ready to prove Theorem \ref{eq:mBesoApproxSpline}.
\begin{proof}[Proof of Theorem \ref{eq:mBesoApproxSpline}]
For $k=(k_1,\dots,k_d) \in \Natural_+^d$, let 
$
P_{k_i}^{(i)} f(x) 
$ be the function operating $P_k$ defined in \eqref{eq:PkDecomposition}
to $f$ as a function of $x_i$ with other components $x_j~(j \neq i)$
fixed, and 
let 
\begin{align}
p_k := \prod_{i=1}^d (P_{k_i}^{(i)} - P_{k_{i}-1}^{(i)})f.
\label{eq:pkdefinition}
\end{align}
Then, $p_k$ can be expressed as $p_k(x) = \sum_{j \in J_m^d(k) }\alpha_{k,j} M_{k,j}^d(x)$.

Let $T_{k_i}^{(i)} = I - P^{(i)}_{k_i}$ and $\|f\|_{p,i}$ be the $L^p$-norm of $f$ 
as a function of $x_i$ with other components $x_j~(j \neq i)$ fixed
(i.e., if $p < \infty$, $\|f\|_{p,i}^p = \int |f(x)|^p \dd x_i$), then \Eqref{eq:fPkwkpIneq} gives 
$$
\|T_{k_i}^{(i)} f \|_{p,i} \lesssim \sup_{|h_i| \leq 2^{-k_i}} \|\Delta_{h_i}^{r,i}(f) \|_{p,i}.
$$
Thus, by applying the same argument again, it also holds 
\begin{align*}
\| \|T_{k_i}^{(i)}T_{k_j}^{(j)} f \|_{p,i} \|_{p,j}
& \lesssim \| \sup_{|h_i| \leq 2^{-k_i}} \|\Delta_{h_i}^{r,i}(T_{k_j} f) \|_{p,i} \|_{p,j}\\
& = \sup_{|h_i| \leq 2^{-k_i}} \|\| T_{k_j}  \Delta_{h_i}^{r,i}(f) \|_{p,j}\|_{p,i} ~~~(\text{$\because$ the definition of $\Delta_{h_i}^{r,i}$ and Fubini's theorem})\\
& \lesssim \sup_{|h_i| \leq 2^{-k_i}}  \sup_{|h_j| \leq 2^{-k_j}} \| \| \Delta_{h_i}^{r,i}(\Delta_{h_j}^{r,j} (f)) \|_{p,j}\|_{p,i}, 
\end{align*}
for $i \neq j$. Thus, applying the same argument recursively, for $u \subset [d]$, it holds that 
$$
\left\| \prod_{i \in u}T_{k_i}^{(i)} f \right\|_p \lesssim w_{r,p}^u(f,2^{-k})
$$
for $k \in \Natural_+^d(u)$. 
Therefore, since $p_k = \prod_{i=1}^d (T_{k_i - 1}^{(i)} - T_{k_{i}}^{(i)})f 
= \sum_{u \subset [d]} (-1)^{|u|} \left( \prod_{i \in u} T_{k_i}^{(i)} \prod_{i \not \in u} T_{k_{i-1}}^{(i)}  \right)f$,
by letting $e = \{i \mid k_i > 0\}$, we have that  
$$
\|p_k\|_p \lesssim \sum_{u \subset [d]} 
\left\|\left( \prod_{i \in u} T_{k_i}^{(i)} \prod_{i \not \in u} T_{k_{i-1}}^{(i)}  \right)f \right\|_{p}
\lesssim
\sum_{u \subset [d]}  w_{r,p}^{\hat{u}}(f,2^{-(k^u)_{\hat{u}}})
\lesssim 
\sum_{e \subset u}  w_{r,p}^{u}(f,2^{-k^u})
$$
where $k^u_i := k_i~(i \in u)$ and $k^u_i := k_i -1~(i \not \in u)$,
$\hat{u} = \{ i \mid k^u_i \geq 0\}$, and $(k^u)_{\hat{u}}$ is a vector such that $(k^u)_{\hat{u},i} = k^u_i$ for $i \in \hat{u}$
and $(k^u)_{\hat{u},i} = 0$ for $i \not \in \hat{u}$.
Now let 
$$
\|(p_k)_k\|_{b_q^{\alpha}(L^p)}  
= \Bigg(\sum_{k \in \Natural_+^d} \big (2^{\alpha\|k\|_1} \|p_k\|_{L^p} \big)^q\Bigg)^{1/q}
$$ 
for $p_k \in L^p(\Omega) ~(k \in \Natural_+^d)$.
Hence, if we set $a_k = \sum_{e \subset u}  w_{r,p}^{u}(f,2^{-k^u})$ for $k \in \Natural^d$
and $e = \{ i \mid k_i > 0\}$, 
we have that $$\|(p_k)_k\|_{b_q^{\alpha}(L^p)} \lesssim \|(a_k)\|_{b_q^{\alpha}(L^p)} \simeq \|f\|_{MB_{p,q}^s}.$$
On the other hand, following the same line of Theorem 2.1 (ii) of \cite{Complexity:Dung:2011},
we also obtain the opposite inequality $\|f\|_{MB_{p,q}^s} \simeq \|(a_k)\|_{b_q^{\alpha}(L^p)}  \lesssim \|(p_k)_k\|_{b_q^{\alpha}(L^p)}$
(note that the analogous inequality to Lemma 2.3 of \cite{Complexity:Dung:2011}
also holds in our setting by replacing $q_s$ with $p_s$ and $\omega_r^e(f,2^{-k})_p$ by $w_{r,p}^e$).

Therefore, $f \in MB_{p,q}^\alpha$ if and only if
$(p_k)_{k \in \Natural_+^d}$ given by \Eqref{eq:pkdefinition} satisfies $\|(p_k)_k\|_{b_q^{\alpha}(L^p)} < \infty$
and 
$f$ can be decomposed into 
$
f = \sum_{k \in \Natural_+^d} p_k
$
where convergence is in $MB_{p,q}^\alpha$.
Moreover, it holds that $\|f\|_{MB^\alpha_{p,q} } \simeq \|(p_k)_k\|_{b_q^{\alpha}(L^p)}$.
This can be shown by Theorem 2.1 of \cite{Complexity:Dung:2011}.
Moreover, by the quasi-norm equivalence $\|p_k\|_p \simeq 2^{-\|k\|_1/ p} (\sum_{j \in J_m^d(k)} |\alpha_{k,j}|^p)^{1/p}$,
we also have $\|(\alpha_{k,j})_{k,j}\|_{mb_{p,q}^\alpha} \simeq \|f\|_{MB^\alpha_{p,q} } $.

If $p \geq r$, the assertion can be shown in the same manner
as Theorem 3.1 of \cite{Complexity:Dung:2011}.

For the setting of $p < r$, we need to use an adaptive approximation method.
In the following, we assume $p < r$.
For a given $K$, by choosing $K^*$ appropriately later, 
we set 
$$
R_K(f)(x) = \sum_{k \in \Natural_+^d: \|k\|_1 \leq K} p_k
+ \sum_{k \in \Natural_+^d: K < \|k\|_1 \leq K^*} G_k(p_k),
$$
where $G_k(p_k)$  is given as 
$$
G_k(p_k)  = \sum_{1 \leq i \leq n_k} \alpha_{k,j_i} M_{k,j_i}^d(x)
$$
where $(\alpha_{k,j_i})_{i=1}^{|J_m^d(k)|}$
is the sorted coefficients in decreasing order of their absolute value: 
$|\alpha_{k,j_1}| \geq |\alpha_{k,j_2}|
\geq \dots \geq |\alpha_{k,j_{|J_m^d(k)|}}|$.
Then, it holds that 
$$
\|p_k - G_k(p_k)\|_{r} \leq \|p_k\|_{p} 2^{\delta \|k\|_1} n_k^{-\delta},
$$
where $\delta := (1/p - 1/r)$ (see the proof of Theorem 3.1 of \cite{dung2011optimal}
and Lemma 5.3 of \cite{Complexity:Dung:2011}).
Moreover, we also have 
$$
\|p_k \|_{r} \leq \|p_k\|_{p} 2^{\delta \|k\|_1} 
$$
for $k \in \Natural_+^d$ with $\|k\|_1 > K^*$.

Here, we define  $N$ as 
$$
N =  \lceil \log_2(K) \rceil.
$$
Let $\epsilon = (\alpha - \delta)/(2\delta)$, and 
$$
K^* = \lceil K(1 + 1/\epsilon) \rceil, 
$$
and $n_k = \lceil 2^{K -\epsilon(\|k\|_1 - K)}\rceil$ for $k \in \Natural_+^d$ with $K+1 \leq \|k\|_1 \leq  K^*$.

Then, by Lemma 5.3 of \cite{Complexity:Dung:2011}, we have that 
\begin{align}
\|f - R_K(f)\|_{L^r}^r 
& \lesssim \sum_{K  < \|k\|_1 \leq K^*} 
 \|p_k - G_k(p_k)\|_{L^r}^r
+ 
\sum_{K^*  < \|k\|_1 }  \|p_k \|_{L^r}]^r
\notag
\\
& 
\lesssim \sum_{K  < \|k\|_1 \leq K^*} 
[ \|p_k\|_{p} 2^{\delta \|k\|_1} n_k^{- \delta}]^r
+ 
\sum_{K^*  < \|k\|_1 } 
[2^{\delta \|k\|_1} \|p_k \|_{L^p}]^r.
\label{eq:fRNsubtraction}
\end{align}

In the following, we require an upper bound of 
${k + d-1 \choose d-1 }$.
Hence, we evaluate this quantity beforehand. 
This can be upper bounded 
by the Stering's formula as 
$${k + d-1 \choose d-1 } \leq \frac{\sqrt{2}e}{2 \pi} 
\underbrace{\left(1 + \frac{d-1}{k}\right)^k \left(1 + \frac{k}{d-1}\right)^{d-1}}_{= D_{k,d}} \leq D_{k,d}.$$
Let $\xi > 0$ be a positive real number satisfying  $1 + \xi \geq K^*/K$. 
We can see that $\xi$ can be chosen as $\xi =  1/\epsilon + o(1)$.
Then, we have that 
\begin{align}
D_{K^*,d} & = D_{K,d}
\frac{(1 + \frac{d-1}{K^*})^{K^*}}{(1 + \frac{d-1}{K})^{K}}
\frac{(1 + \frac{K^*}{d-1})^{d-1}}{(1 + \frac{K}{d-1})^{d-1}}
\leq 
D_{K,d} \frac{(1 + \frac{d-1}{K^*})^{K^*}}{(1 + \frac{d-1}{K^*})^{K}}
\left(\frac{1}{1 + \frac{K}{d-1}} + \frac{K^*}{(d-1)
(1 + \frac{K}{d-1})}\right)^{d-1} \notag \\
&
\leq
D_{K,d} \left(1 + \frac{d-1}{K^*}\right)^{K^*-K}
\left(\frac{d -1 + K^*}{d-1 + K}\right)^{d-1} 
=
D_{K,d} 
\left(1 + \frac{d-1}{K}\right)^{\xi K}
\left(1 + \xi\right)^{d-1} \notag \\
& \leq 
D_{K,d} 
e^{(d-1)\xi} 
(1 + \xi)^{d-1} \simeq D_{K,d}.
\label{eq:Dkd_upperbound}
\end{align}

(a) Suppose that $q \leq r$ and $r<\infty$. Then 
\begin{align*}
& \|f - R_K(f)\|_{L^r}^q =\|f - R_K(f)\|_{L^r}^{r \frac{q }{r}}  \\
& \lesssim \left\{ \sum_{K  < \|k\|_1 \leq K^*} 
[2^{\delta \|k\|_1 } n_k^{-\delta} \|p_k \|_{L^p}]^r
+ 
\sum_{K^*  < \|k\|_1 } 
[2^{\delta \|k\|_1} \|p_k \|_{L^p}]^r \right\}^{\frac{q}{r}} ~~~~~~~~(\text{$\because$ \Eqref{eq:fRNsubtraction}}) \\
& \lesssim \sum_{K  < \|k\|_1 \leq K^*} 
[2^{\delta \|k\|_1 } n_k^{-\delta} \|p_k \|_{L^p}]^q
+ 
\sum_{K^*  < \|k\|_1 } 
[2^{\delta \|k\|_1} \|p_k \|_{L^p}]^q \\
& \leq
N^{- \delta q} 2^{-(\alpha - \delta)K q}
\sum_{K  < \|k\|_1 \leq K^*} 
[\underbrace{2^{-(\alpha - \delta - \delta \epsilon)( \|k\|_1 - K)}}_{\leq 1} 2^{\alpha \|k\|_1} \|p_k \|_{L^p}]^q
+
2^{-q (\alpha - \delta) K^*} 
\sum_{K^*  < \|k\|_1 } [2^{\alpha \|k\|_1} \|p_k \|_{L^p}]^q \\
& \lesssim
(N^{- \delta} 2^{-(\alpha - \delta)K} 
+2^{- (\alpha - \delta) K^*} )^q
\|f\|_{MB^\alpha_{p,q}}^q \\
& \leq
(N^{- \alpha} )^q
\|f\|_{MB^\alpha_{p,q}}^q.
\end{align*}
%

(b) Suppose that $q > r$ and $r < \infty$. Then,
letting $\nu = q/r ( > 1)$ and $\nu' = 1/(1-1/\nu) = q/(q - r)$, we have 
\begin{align*}
& \|f - R_K(f)\|_{L^r}^r
\lesssim \sum_{K  < \|k\|_1 \leq K^*} 
[2^{\delta \|k\|_1 } n_k^{-\delta} \|p_k \|_{L^p}]^r
+ 
\sum_{K^*  < \|k\|_1 } 
[2^{\delta \|k\|_1} \|p_k \|_{L^p}]^r ~~~~~(\text{$\because$ \Eqref{eq:fRNsubtraction}})\\
& \leq
N^{- \delta r} 2^{-(\alpha - \delta)K r}
\sum_{K  < \|k\|_1 \leq K^*} 
[2^{-(\alpha - \delta - \delta \epsilon)( \|k\|_1 - K)} 2^{\alpha \|k\|_1} \|p_k \|_{L^p}]^r
+
\sum_{K^*  < \|k\|_1 } [2^{\alpha \|k\|_1} \|p_k \|_{L^p}]^r
(2^{- (\alpha - \delta)\|k\|_1} )^r
 \\
& \leq
(N^{- \delta} 2^{-(\alpha - \delta)K} 
+2^{- (\alpha - \delta) K^*} )^r
\Big\{
\sum_{K  < \|k\|_1 \leq K^*} 
[2^{-(\alpha - \delta - \delta \epsilon)( \|k\|_1 - K)} 2^{\alpha \|k\|_1} \|p_k \|_{L^p}]^r  \\
& +
\sum_{K^*  < \|k\|_1 } [2^{\alpha \|k\|_1} \|p_k \|_{L^p}]^r
2^{-(\alpha -\delta)(\|k\|_1- K^*)r}
\Big\}
\\
& \leq 
(N^{- \delta} 2^{-(\alpha - \delta)K} 
+2^{- (\alpha - \delta) K^*} )^r
\left\{
\sum_{K  < \|k\|_1 \leq K^*} 
[ 2^{\alpha \|k\|_1} \|p_k \|_{L^p}]^{r\nu} 
+
\sum_{K^*  < \|k\|_1 } [2^{\alpha \|k\|_1} \|p_k \|_{L^p}]^{r\nu}
\right\}^{1/\nu}\\
&
\times
\left\{
\sum_{K  < \|k\|_1 \leq K^*} 
[2^{-(\alpha - \delta - \delta \epsilon)( \|k\|_1 - K) }]^{r \nu'} 
+
\sum_{K^*  < \|k\|_1 } [2^{-(\alpha -\delta)(\|k\|_1 - K^*) }]^{r \nu'}
\right\}^{1/\nu'}
\\
& \lesssim
(N^{- \delta} 2^{-(\alpha - \delta)K} 
+2^{- (\alpha - \delta) K^*} )^r
\|f\|_{MB^\alpha_{p,q}}^r  D_{K,d}^{r(1/r - 1/q)}
~~~~~~~~~(\because \text{\Eqref{eq:Dkd_upperbound}})
\\
& \lesssim
(N^{- \alpha} D_{K,d}^{1/r - 1/q})^r \|f\|_{MB^\alpha_{p,q}}^r.
\end{align*}

(c) Suppose that $r=\infty$.
Then, similarly to the analysis in (b), we can evaluate 
\begin{align*}
& \|f - R_K(f)\|_{L^r} \\
& \lesssim 
N^{- \delta} 2^{-(\alpha - \delta)K}
\sum_{K  < \|k\|_1 \leq K^*} 
[2^{-(\alpha - \delta - \delta \epsilon)( \|k\|_1 - K)} 2^{\alpha \|k\|_1} \|p_k \|_{L^p}]
+
\sum_{K^*  < \|k\|_1 } [2^{\alpha \|k\|_1} \|p_k \|_{L^p}]
(2^{- (\alpha - \delta)\|k\|_1} ) \\
& \lesssim
(N^{- \delta} 2^{-(\alpha - \delta)K} 
+2^{- (\alpha - \delta) K^*} )
 D_{K,d}^{(1 - 1/q)_+}
\|f\|_{MB^\alpha_{p,q}} \\
&  \lesssim
N^{- \alpha} D_{K,d}^{(1 - 1/q)_+}
\|f\|_{MB^\alpha_{p,q}} .
\end{align*}
\end{proof}

\section{Proofs of Theorems \ref{thm:EstimationErrorNNBesov} and \ref{thm:EstimationErrorNN}}

\label{sec:ProofsOfEstimationErrorBounds}

\begin{proof}[Proof of Theorem \ref{thm:EstimationErrorNNBesov}]
We use Proposition \ref{prop:RiskBoundCovering}.
We just need to evaluate the covering number of $\hat{\calF} = 
\{\fbar \mid f \in \Psi(L,W,S,B)\}$ for $(L,W,S,B)$ given in Theorem \ref{eq:mBesovApproxByNN}
where $\fbar$ is the clipped function for a given $f$.
Note that the covering number of $\hat{\calF}$ is not larger than that of $\Psi(L,W,S,B)$.
Thus, we may evaluate that of $\Psi(L,W,S,B)$.
From Lemma \ref{lemm:CovNPhi}, the covering number is obtained as
$$
\log N(\delta,\hat{\calF},\|\cdot\|_\infty)
\lesssim N  [\log(N)^2 + \log(\delta^{-1}) ].
$$

From Proposition \ref{prop:BesovApproxByNN}, it holds that 
$$
\|\ftrue - R_K(\ftrue)\|_2 \lesssim N^{-s/d}.
$$
Note that 
$$
\|f - \ftrue\|_{\LPi(P_X)}^2 \lesssim \|f - \ftrue\|_2^2.
$$
for any $f:[0,1]^d \to \Real$ because $p(x) \leq R$.
Therefore, by applying Proposition \ref{prop:RiskBoundCovering} with $\delta = 1/n$, we have that 
\begin{align}
\EE_{D_n}[\|\fhat - \ftrue\|_{\LPi(P_X)}^2]
\lesssim N^{-2 s/d} + \frac{N (\log(N)^2 + \log(n)) }{n} + \frac{1}{n}.
\label{eq:DnHatBound}
\end{align}
Here, the right hand side is minimized by setting $N \asymp n^{\frac{d}{2s + d}}$ up to $\log(n)^2$-order, and then have an upper bound 
of the RHS as 
$$
n^{- \frac{2s}{2s + d}} \log(n)^2.
$$
This gives the assertion.
\end{proof}

\begin{proof}[Proof of Theorem \ref{thm:EstimationErrorNN}]
The proof follows the almost same line as the proof of Theorem \ref{thm:EstimationErrorNNBesov}.
By noting $S = O(2^{K} D_{K,d})$, $L = O(K)$ and $W = O(2^{K} D_{K,d})$, Lemma \ref{lemm:CovNPhi} gives an upper bound of 
the covering number as 
$$
\log N(\delta,\hat{\calF},\|\cdot\|_\infty)
\lesssim 2^{K} D_{K,d} [ K \log(2^{K} D_{K,d}) + \log(\delta^{-1}) ]
\lesssim 2^{K} D_{K,d} (K^2 + \log(1/\delta)).
$$

Letting $r = 2$, we have that 
$$
\|\ftrue - R_K(\ftrue)\|_2 \lesssim 2^{-s K} D_{K,d}^u
$$
where $u = (1 - 1/q)_+$ for $p \geq 2$ and $u = (1/2 - 1/q)_+$ for $p < 2$.



Then, by noting that 
$$
\|f - \ftrue\|_{\LPi(P_X)}^2 \lesssim \|f - \ftrue\|_2^2,
$$
for any $f:[0,1]^d \to \Real$, and 
by applying Proposition \ref{prop:RiskBoundCovering} with $\delta = 1/n$, we have that 
\begin{align}
\EE_{D_n}[\|\fhat - \ftrue\|_{\LPi(P_X)}^2]
\lesssim 2^{-2 s K}D_{K,d}^{2u} + \frac{2^K D_{K,d} (K^2 + \log(\delta^{-1})) }{n} + \frac{1}{n}.
\label{eq:DnHatBound}
\end{align}

Here, we use the following evaluations for $D_{K,d}$: (a) $D_{K,d} \lesssim K^{d-1}$, and (b) $D_{K,d} \lesssim [e(1 + \frac{d}{K})]^{K}$.

(a) For the evaluation, $D_{K,d} \lesssim K^{d-1}$, we have an upper bound of 
the right hand side of \Eqref{eq:DnHatBound} as 
$$
2^{-2sK} K^{2u(d-1)} + \frac{2^K K^{d-1}(K^2 + \log(n)) }{n},
$$
which is minimized by setting 
$K = \lceil \frac{1}{1+2 s} \log_2(n) + \frac{(2u-1)(d - 1)}{1+2 s} \log_2 \log(n) \rceil$ 
up to $\log\log(n)$-order. In this situation, we have the generalization error bound as 
$$
n^{- \frac{2s}{2s+1}} \log(n)^{\frac{2 (d-1)(u + s)}{1+2s}} \log(n)^2.
$$


(b) For the evaluation, $D_{K,d} \lesssim [e(1 + \frac{d}{K})]^{K} \leq e^K e^d$, \Eqref{eq:DnHatBound} gives an upper bound of 
$$
2^{-2 sK} e^{2uK} + \frac{2^K e^K (K^2 + \log(n)) }{n}.
$$
Then, the right hand side is minimized by 
$K =\lceil \frac{1}{1+2s + (1-2u)\log_2(e)} \log_2(n) \rceil$.
Then, we have that 
$$
n^{-\frac{2s - 2 u \log_2(e)}{1+2s+(1-2u) \log_2(e)}}  \log(n)^2.
$$
This gives the assertion.
\end{proof}

\section{Minimax optimality}
\label{sec:MinimaxMixedSmooth}

\begin{proof}[Proof of Theorem \ref{eq:MinimaxOptimalboundOfmBesov}]
First note that since $P_X$ is the uniform distribution, it holds that $\|\cdot\|_{\LPiPx} = \|\cdot\|_{\LPi([0,1]^d)}$.
The $\epsilon$-covering number $\calN(\epsilon,\calG,\LPiPx)$ with respect to $\LPiPx$
for a function class $\calG$ 
is the minimal number of balls with radius $\epsilon$ measured by $\LPiPx$-norm needed to cover the set $\calG$ \citep{Book:VanDerVaart:WeakConvergence}.
The $\delta$-packing number $\calM(\delta,\calG,L^2(P_X))$ of a function class $\calG$ with respect to $\LPiPx$ norm is the largest number of functions $\{f_1, \dots, f_{\calM} \} \subseteq \calG$
such that $\|f_i - f_j\|_{\LPiPx} \geq \delta$ for all $i\neq j$.
It is easily checked that 
\begin{equation}
\calN(\delta/2,\calG,\LPiPx) \leq \calM(\delta,\calG,\LPiPx) \leq \calN(\delta,\calG,\LPiPx).
\label{eq:NMrelation}
\end{equation}

For a given $\delta_n > 0$ and $\varepsilon_n > 0$, 
let $Q$ be the $\delta_n$ packing number $\calM(\delta_n,U(MB_{p,q}^s),\LPi(P_X))$ of $U(MB_{p,q}^s)$ and 
$N$ be the $\varepsilon_n$ covering number of that. 
\cite{JMLR:Raskutti+Martin:2012} utilized the techniques developed by 
\cite{AS:Yang+Barron:99} to show the following inequality in their proof of Theorem 2(b) : 
\begin{align*}
\inf_{\fhat} \sup_{\fstar \in U(MB_{p,q}^s)}\EE_{D_n}[\|\fhat - \fstar \|_{\LPiPx}^2] 
& \geq 
\inf_{\fhat} \sup_{\fstar \in U(MB_{p,q}^s)} \frac{\delta_n^2}{2} P[\|\fhat - \fstar \|_{\LPiPx}^2 \geq \delta_n^2/2] \\
& \geq \frac{\delta_n^2}{2}\left(1-\frac{\log(N) + \frac{n}{2\sigma^2}\varepsilon_n^2 + \log(2)}{\log(Q)}\right).
\end{align*}
Thus by taking $\delta_n$ and $\varepsilon_n$ to satisfy 
\begin{align}
\frac{n}{2\sigma^2}\varepsilon_n^2 &\leq \log(N), \label{eq:epsilonnbound}  \\ 
8 \log(N) &\leq \log(Q),   \label{eq:NQbound} \\
4 \log(2) &\leq \log(Q),   \label{eq:2Qbound} 
\end{align}
the minimax rate is lower bounded by $\frac{\delta_n^2}{4}$.
This can be achieved by properly setting $\varepsilon_n \simeq \delta_n$.
Now, for given $N$ with respect to $\delta_n > 0$, $M = \log(N)$ satisfies
$$
\delta_n \gtrsim M^{-s} \log(M)^{(d-1)(s+ 1/2 - 1/q)_+}
$$
(Theorem 6.24 of \cite{dung2016hyperbolic}).
Hence, it suffices to take 
\begin{align}
\label{eq:epsboundfinal}
& M \simeq n^{\frac{1}{2s+1}} \log(n)^{\frac{2(d-1)(s + 1/2 - 1/q)_+}{2s+1}}, \\
& \varepsilon_n \simeq \delta_n \simeq n^{- \frac{2s}{2s+1}} \log(n)^{\frac{2(d-1)(s + 1/2 - 1/q)_+}{2s+1}},
\end{align}
which gives the assertion.
\end{proof}

\section{Auxiliary lemmas}

Let the $\epsilon$-covering number 
with respect to $\LPiPx$
for a function class $\calG$ 
be $\calN(\epsilon,\calG,\LPiPx)$
as defined in the proof of Theorem \ref{eq:MinimaxOptimalboundOfmBesov}.

\begin{Proposition}[\cite{2017arXiv170806633S}]
\label{prop:RiskBoundCovering}
Let $\calF$ be a set of functions. Let $\fhat$ be any estimator in $\calF$. Define 
$$
\Delta_n := \EE_{D_n}\left[\frac{1}{n}\sum_{i=1}^n (y_i - \fhat(x_i))^2 - \inf_{f\in \calF} \frac{1}{n}\sum_{i=1}^n (y_i - f(x_i))^2 \right].
$$
Assume that $\|\ftrue \|_\infty \leq F$ and all $f \in \calF$ satisfies $\|f\|_\infty \leq F$ for some $F \geq 1$.
If $0 <  \delta < 1$ satisfies $\calN(\delta,\calF,\|\cdot\|_\infty) \geq 3$, then
there exists a universal constant $C$ such that 
\begin{align*}
& \EE_{D_n}[\|\fhat - \ftrue\|_{\LPi(P_X)}^2]  \\
& \leq C  (1+\epsilon)^2 \left[ \inf_{f \in \calF} \|f-\ftrue\|_{\LPi(P_X)}^2 + F^2 \frac{\log \calN(\delta,\calF,\|\cdot\|_\infty) - \log(\delta)}{n \epsilon} 
+ \delta F^2
+ \Delta_n \right],
\end{align*}
for any $\epsilon \in (0,1]$.

\end{Proposition}
\begin{proof}[Proof of Proposition \ref{prop:RiskBoundCovering}]
This is almost direct consequence of Lemma 8 of \cite{2017arXiv170806633S}\footnote{We noticed that there exit some 
technical flaws in the proof of the lemma, e.g., an incorrect application of the uniform bound to derive the risk of an estimator. 
However, these flaws can be fixed and the statement itself holds with a slight modification.
Accordingly, there appears $-\log(\delta)$ term and $\delta F$ is replaced by $\delta F^2$}.
The only difference is the assumption of $\|f\|_\infty \leq F$ for $f \in \calF$
and $f = \ftrue$ while Lemma 8 of \cite{2017arXiv170806633S} assumed 
$ 0 \leq f(x) \leq F'$ for $F' > 1$.
However, this can be easily fixed by shifting the function value by $+ F$
then the range of $f$ is modified to $[0,2F]$.
Then, our situation is reduced to that of Lemma 8 of \cite{2017arXiv170806633S}
by substituting $F' \leftarrow 2F$.
\end{proof}

\begin{Lemma}[Covering number evaluation]
\label{lemm:CovNPhi}
The covering number of $\Phi(L,W,S,B)$ can be bounded by 
\begin{align*}
\log \calN(\delta,\Phi(L,W,S,B),\|\cdot\|_\infty) 
& \leq S \log(\delta^{-1} L (B \vee 1)^{L -1} (W+1)^{2L} ) \\
& \leq 2 S L \log(\delta^{-1} L (B \vee 1)  (W+1) ).
\end{align*}

\end{Lemma}

\begin{proof}[Proof of Lemma \ref{lemm:CovNPhi}]

Given a network $f \in \Phi(L,W,S,B)$ expressed as 
$$
f(x) =  (\Well{L} \eta( \cdot) + \bell{L}) \circ \dots  
\circ (\Well{1} x + \bell{1}),
$$
let 
$$
\calA_k(f)(x) = \eta \circ (\Well{k-1} \eta( \cdot) + \bell{k-1}) \circ \dots  \circ (\Well{1} x + \bell{1}),
$$
and 
$$
\calB_k(f)(x) =  (\Well{L} \eta( \cdot) + \bell{L}) \circ \dots  \circ (\Well{k} \eta(x) + \bell{k}),
$$
for $k=2,\dots,L$.
Corresponding to the last and first layer, we define $\calB_{L+1}(f)(x) = x$ and $\calA_{1}(f)(x) = x$.
Then, it is easy to see that $f(x) = \calB_{k+1}(f) \circ (\Well{k} \cdot + \bell{k}) \circ \calA_{k}(f)(x)$.
Now, suppose that a pair of different two networks $f, g \in\Phi(L,W,S,B)$ given by 
$$
f(x) = (\Well{L} \eta( \cdot) + \bell{L}) \circ \dots  
\circ (\Well{1} x + \bell{1}),~~
g(x) = ({\Well{L}}' \eta( \cdot) + {\bell{L}}') \circ \dots  
\circ ({\Well{1}}' x + {\bell{1}}'),
$$
has a parameters with distance $\delta$: $\|\Well{\ell} - {\Well{\ell}}'\|_\infty \leq \delta$ and $\|\bell{\ell} - {\bell{\ell}}'\|_\infty \leq \delta$.
Now, not that $\|\calA_k(f) \|_\infty  \leq \max_j \|\Well{k-1}_{j,:}\|_1 \|\calA_{k-1}(f) \|_\infty + \|\bell{k-1}\|_\infty
\leq W B \|\calA_{k-1}(f) \|_\infty + B \leq (B \vee 1) (W +1 )\|\calA_{k-1}(f) \|_\infty 
\leq (B \vee 1)^{k-1} (W +1 )^{k-1}$,
and similarly the Lipshitz continuity of $\calB_k(f)$ with respect to $\|\cdot\|_\infty$-norm is bounded as 
$
(B W)^{L - k + 1}.
$
Then, it holds that 
\begin{align*}
& |f(x) - g(x)| \\
= & \left|\sum_{k=1}^L \calB_{k+1}(g) \circ (\Well{k} \cdot + \bell{k}) \circ \calA_{k}(f)(x)
 - \calB_{k+1}(g) \circ ({\Well{k}}' \cdot + {\bell{k}}') \circ \calA_{k}(f)(x) \right| \\
\leq & \sum_{k=1}^L  (B W)^{L - k } \|
(\Well{k} \cdot + \bell{k}) \circ \calA_{k}(f)(x) - ({\Well{k}}' \cdot + {\bell{k}}') \circ \calA_{k}(f)(x)
 \|_\infty \\
\leq &  \sum_{k=1}^L  (B W)^{L - k }  \delta [W  (B \vee 1)^{k-1} (W +1 )^{k-1}  + 1] \\
\leq &  \sum_{k=1}^L  (B W)^{L - k }  \delta  (B \vee 1)^{k-1} (W +1 )^{k}
\leq \delta L (B \vee 1)^{L -1} (W+1)^{L}.
\end{align*}
Thus, for a fixed sparsity pattern (the locations of non-zero parameters), 
the covering number is bounded by $ \left( \delta/ [L (B \vee 1)^{L -1} (W+1)^{L}] \right)^{-S}$.
There are the number of configurations of the sparsity pattern is bounded by 
${(W+1)^L \choose S} \leq  (W+1)^{LS}$. Thus, the covering number of the whole space $\Phi$ is bounded as
$$
 (W+1)^{LS} \left\{ \delta/ [L (B \vee 1)^{L -1} (W+1)^{L}] \right\}^{-S}
 = [\delta^{-1} L (B \vee 1)^{L -1} (W+1)^{2L} ]^S,
$$
which gives the assertion.

\end{proof}

\end{document}